\documentclass[lettersize,onecolumn]{IEEEtran}

\usepackage{algorithmic}
\usepackage{algorithm}
\usepackage{array}
\usepackage[caption=false,font=normalsize,labelfont=sf,textfont=sf]{subfig}
\usepackage{textcomp}
\usepackage{stfloats}
\usepackage{url}
\usepackage{verbatim}
\usepackage{graphicx}
\usepackage{cite}
\usepackage{mathtools}
\usepackage{thm-restate} 
\usepackage[utf8]{inputenc} 
\usepackage[T1]{fontenc}    
\usepackage{url}            
\usepackage{booktabs}       
\usepackage{amsfonts}       
\usepackage{nicefrac}       
\usepackage{microtype}      
\usepackage{tikz}
\usepackage{mymacros}
\usepackage{pgfplots}

\usepackage{subfig}
\pgfplotsset{compat=1.18}

\usepackage{enumitem}

\allowdisplaybreaks

\usepackage{hyperref} 
\usepackage{graphicx}%
\usepackage{multirow}%
\usepackage{amsmath,amssymb,amsfonts}%
\usepackage{amsthm}%
\usepackage{mathrsfs}%
\usepackage{xcolor}%
\usepackage{textcomp}%
\usepackage{manyfoot}%
\usepackage{booktabs}%
\usepackage{listings}%

\newtheorem{ass}{Assumption}[section]

\newtheorem{definition}{Definition}[section]%

\usepackage[square,numbers]{natbib}
\begin{document}

\title{Sliding-Window Thompson Sampling for Non-Stationary Settings}

\author{Marco Fiandri, Alberto Maria Metelli, Francesco Trovò}
\markboth{}%
{}

\IEEEpubid{}

\maketitle
\begin{abstract}
\emph{Non-stationary multi-armed bandits} (NS-MABs) model sequential decision-making problems in which the expected rewards of a set of actions, a.k.a.~arms, evolve over time. 
In this paper, we fill a gap in the literature by providing a novel analysis of \emph{Thompson sampling}-inspired (TS) algorithms for NS-MABs that both corrects and generalizes existing work. Specifically, we study the cumulative frequentist regret of two algorithms based on \emph{sliding-window} TS approaches with different priors, namely $\texttt{Beta-SWTS}$ and $\texttt{$\gamma$-SWGTS}$.
We derive a \emph{unifying} regret upper bound for these algorithms that applies to \emph{any} arbitrary NS-MAB (with either Bernoulli or subgaussian rewards). Our result introduces new indices that capture the inherent sources of complexity in the learning problem. Then, we specialize our general result to two of the most common NS-MAB settings: the \textit{abruptly changing} and the \textit{smoothly changing} environments, showing that it matches state-of-the-art results.
Finally, we evaluate the performance of the analyzed algorithms in simulated environments and compare them with state-of-the-art approaches for NS-MABs.
\end{abstract}


\section{Introduction}
A \emph{multi-armed bandit} \citep[MAB,][]{lattimore2020bandit} problem is a sequential game between a learner and an environment. In each round $t$, the learner first chooses an action, often called \emph{arm}, and the environment then reveals a \emph{reward}. The goal of the learner is to balance exploration and exploitation, minimizing the \emph{expected cumulative regret},
defined as the performance difference, expressed in expected rewards, between playing the optimal arm and the learner. These algorithms have traditionally been studied in \emph{stationary} settings where the environment does not change over time. As a consequence, the optimal arm $i^*$ is constant and does not depend on the round $t$. However, many real-world applications, such as online advertising \citep{pandey2006ad, kawale2015efficient}, healthcare \citep{lu2021banditalgorithmsprecisionmedicine, dasgupta2024bayesian, dixit2023thompson,jaiswal2025deconfoundedwarmstartthompsonsampling} and dynamic pricing \citep{ganti2018thompson, bi2024personalized}, operate in environments that are changing over time. These are often referred to as \emph{non-stationary} MABs (NS-MABs), where the world evolves independently of the actions taken by the learner. As a consequence, the optimal arm $i^*(t)$ is potentially different in every round $t$, making the decision problem more challenging. This requires the design of learning algorithms able to \emph{adapt} to environment modifications.

In the past years, the bandit literature focused on the design of algorithms that handle \emph{specific classes} of NS-MABs characterized by certain {regularity conditions}. The \emph{piecewise-constant abruptly changing} MABs \citep{garivier2008upper, re2021exploiting,liu2018change,allesiardo2017non,besbes2014stochastic,besson2019generalized} are characterized by expected rewards that remain constant during some rounds and change at unknown rounds, called \emph{breakpoints}. Another form of regularity are the \emph{smoothly changing} MABs \citep{combes2014unimodal,trovo2020sliding} where the expected rewards vary by a limited amount across rounds. Other forms of regularity include the \emph{rising}  \citep{heidari2016tight,metelli2022stochastic} and \emph{rotting} \cite{seznec2019rotting} MABs, where the expected rewards can only increase or decrease in time, respectively, and the MABs with \emph{bounded variation} \cite{besbes2014stochastic}, where the expected reward is constrained to have a finite cumulative variation over the learning horizon.
Several algorithmic approaches have been adopted for addressing regret minimization in NS-MABs \citep[e.g.,][]{garivier2008upper,combes2014unimodal,besbes2014stochastic,trovo2020sliding}. Among them \emph{Thompson sampling} (TS)~\citep{thompson1933ts} is one of the most widely used bandit algorithms for its simplicity in implementation and its good empirical performance. However, the classical TS algorithm is devised for stationary MABs where they enjoy strong theoretical guarantees \cite{kaufmann2012thompson,agrawal2012analysis,agrawal2017near}. Variations to the classical TS have been proposed to tackle NS-MABs including \emph{sliding-window} \cite{trovo2020sliding} and \emph{discounted} \cite{raj2017taming,qi2025thompson,de2024addressing} approaches. These algorithms come often with theoretical guarantees for specific classes of NS-MABs, namely piecewise-constant abruptly changing and smoothly changing.\footnote{In this paper, following the seminal analysis of TS \citep{agrawal2017near}, we focus on the \emph{frequentist} regret only which represents a more ambitious performance index w.r.t. the \emph{Bayesian} regret \cite{russo2016information}.}



\noindent \textbf{Original Contributions}  In this paper, differently from what is often done in literature, we provide a \emph{unifying analysis} of \emph{sliding-window TS} algorithms that does not rely on the specific form of non-stationarity (namely piecewise-constant abruptly changing and smoothly changing). Our novel analysis shed lights on the inherent complexity of the regret minimization problem in general NS-MABs and introduces new quantities to characterize quantitatively such a complexity. Furthermore, we extend and correct the original analysis of \citet{trovo2020sliding}.\footnote{In Appendix \ref{apx:trovo}, we show that some passages of the analysis by~\citet{trovo2020sliding} are incorrect.} Finally, we show how the state-of-the-art results for the specific forms of non-stationarity (namely piecewise-constant abruptly changing and smoothly changing) can be retrieved as a particular case of our analysis. 
The content of the paper is summarized as follows:
\begin{itemize}[leftmargin=*]
    \item In Section \ref{sec:related}, we survey the related works on TS algorithms and approaches for regret minimization in NS-MABs.
    \item In Section \ref{sec:prelim}, we provide the setting, the assumptions on the reward distributions, and the definition of cumulative regret.
    \item In Section \ref{sec:alg}, we describe two TS-inspired algorithms, namely \texttt{Beta-SWTS} and \texttt{$\gamma$-SWGTS} based on a sliding-window approach, exploiting the $\tau$ (being $\tau$ the window size) most recent samples to estimate the expected rewards.
    \item In the first part of Section \ref{sec:gg}, we introduce new quantities to characterize how complex is to learn with sliding-window algorithms in an NS-MAB with expected rewards evolving with no particular form of non-starionarity. In particular, we define two sets, namely the \emph{learnable set} and the \emph{unlearnable set} (Definition \ref{def:sets}), to describe in which rounds an algorithm exploiting the most recent samples only is expected to identify the optimal arm. Furthermore, we define a new suboptimality gap notion, $\Delta_\tau$ (Definition \ref{def:subgap}) that will be employed in the analysis.
    \item In the second part of Section \ref{sec:gg},  we derive novel \emph{unifying regret upper bounds} of the \texttt{Beta-SWTS} and \texttt{$\gamma$-SWGTS} algorithms described in Section \ref{sec:alg}, for Bernoulli and subgaussian rewards, respectively. Our analysis exploits the quantities previously defined to characterize the complexity of the learning problem and makes no assumption on the underlying form of non-stationarity.
    \item We leverage the results of Section \ref{sec:gg} to derive regret upper bounds for the \emph{abruptly changing} NS-MABs (Section \ref{sec:ac}) and the \emph{smoothly changing} NS-MABs (Section \ref{sec:sc}). Moreover, we show how our bounds are comparable with the state-of-the-art ones derived with analyses tailored for the specific form of non-stationarity.
    \item In Section \ref{sec:exp}, we experimentally compare the performance of the analyzed algorithms with those in the bandit literature that are devised to learn in non-stationary scenarios.
\end{itemize}
The proofs of the results presented in the main paper are reported in Appendix \ref{apx:proofs} and \ref{apx:proofs2}.

\section{Related works}\label{sec:related}

In this section, we survey the main related works about TS and approaches for regret minimization in NS-MABs.

\subsection{Thompson Sampling}

TS was introduced in $1933$~\citep{thompson1933ts} for allocating experimental effort in online sequential decision-making problems, and its effectiveness has been investigated both empirically \citep{chapelle2011empirical,scott2010ts} and theoretically \citep{agrawal2017near,kaufmann2012thompson} only in the last decades. The algorithm has found widespread applications in various fields, including online advertising \citep{graepel2010web,agarwal2013computational,ag2014ad}, clinical trials \citep{aziz2021multi}, recommendation systems \citep{kawale2015efficient} and hyperparameter tuning for machine learning methods \citep{kandasamy2019neuralarchitecturesearchbayesian}. TS is optimal in the stationary case, i.e., achieving instance-dependent regret matching the lower bound \citep{lai1985asymptotically}. However, it has been shown in multiple cases that in NS-MABs~\citep{garivier2011upper,trovo2020sliding,liu2024nonstationarybanditlearningpredictive} or in adversarial settings ~\citep{cesa2006prediction} it provides poor performances in terms of regret.

\subsection{Non-Stationary Bandits}
Lately, \texttt{UCB1} and \texttt{TS} algorithms inspired the development of techniques to face the inherent complexities of NS-MABs \citep{whittle1988restless}. The main idea behind these newly crafted algorithms is to forget past observations, removing samples from the statistics of the arms' expected reward. Two main approaches are present in the bandit literature to forget past observations: \emph{passive} and \emph{active}.
The former iteratively discards the information coming from the far past, making decisions using only the most recent samples coming from the arms selected by the algorithms. Examples of such a family of algorithms are  \texttt{Discounted-TS}~\citep{raj2017taming}, \texttt{DUCB}~\citep{garivier2011upper}, which employ a multiplicative discount factor to reduce the impact of samples seen in the past. It has been shown that these algorithms achieve regret of order $O(\sqrt{\Upsilon_T T}\log(T))$ in piecewise-constant abruptly changing environments, where $\Upsilon_T$ is the number breakpoint present during the learning horizon $T$.
Finally, \texttt{SW-UCB}~\citep{garivier2011upper} used a sliding-window approach in combination with an upper confidence bound to get a regret of order $O(\sqrt{\Upsilon_T T\log(T)})$ in the same setting. 
Instead, the active approach encompasses the use of \emph{change-detection} techniques~\citep{basseville1993detection} to decide when it is the case to discard old samples. This occurs when a sufficiently large change affects the arms' expected rewards. Among the active approaches to deal with the abruptly changing bandits, we mention \texttt{CUSUM-UCB}~\citep{liu2018change} and \texttt{BR-MAB}~\citep{re2021exploiting}. They achieve a regret of order $O \left( \sqrt{\Upsilon_T T \log(\frac{T}{\Upsilon_T})} \right)$. Instead, in the same setting, \texttt{GLR-klUCB}~\citep{besson2019generalized}, based on the use of \texttt{KL-UCB} as a bandit selection algorithm and a nonparametric change point method, achieve an $O(\sqrt{\Upsilon_T T\log(T)})$ regret. Another approach that is worth mentioning is \texttt{RExp3} \citep{besbes2014stochastic}, which builds on \texttt{Exp3} \citep{auer2002finite}, adding scheduled restarts to the original algorithm, and it handles arbitrary evolutions of the expected rewards as long as they are constrained within $[0,1]$ and the learner knows the total variation $V_T$ of the expected reward, providing an $O(V_T^{\frac{1}{3}}T^{\frac{2}{3}})$ regret.
Finally, different approaches to developing TS-like algorithms in NS-MABs resort to de-prioritizing information that more quickly loses usefulness~\citep{liu2024nonstationarybanditlearningpredictive} and deriving a bound on the Bayesian regret of the algorithm.

As a final remark, we point out that differently from \texttt{CUSUM-UCB}, \texttt{GLR-klUCB} and \texttt{BR-MAB}, we are able to characterize the regret for any NS-MAB, as long as the distribution of the rewards is either Bernoulli or subgaussian, and in a more general setting than the piecewise-constant abruptly-changing ones. Furthermore, differently from the analysis of \texttt{RExp3}, we retrieve guarantees on the performance also for expected rewards that are not bounded in $[0,1]$. Moreover, we highlight that in the work by \citet{liu2024nonstationarybanditlearningpredictive}, the authors evaluate the Bayesian regret while we retrieve frequentist bounds on the performance that are notoriously more informative. 
In~\cite{combes2014unimodal}, the authors dealt with non-stationary, smoothly-changing bandits, a setting in which the expected rewards evolve for a limited amount between two rounds. They designed \texttt{SW-KL-UCB} they achieve a $O \left( H(\Delta, T) + \frac{T \log(\tau)}{\Delta^2 \tau} \right)$ regret, where the order of $H(\Delta,T)$ depends on the bandit instance and $\Delta$ is the minimum non-zero distance of the expected rewards within the learning horizon between the best arm and the suboptimal arms. Recently paper \cite{qi2025thompson} analyzed the regret of the $\gamma$-\texttt{SWGTS} algorithm. However, the authors do not face the far more challenging Beta-Binomial case and consider only the piece-wise constant abruptly changing settings.\footnote{We also remark that \cite{qi2025thompson} cite a preprint version of the present paper \citep[][\url{https://arxiv.org/abs/2409.05181}]{fiandri2024sliding}.}

\section{Problem Definition}\label{sec:prelim}
 At each round $t \in \dsb{T}$,\footnote{Let $a,b \in \mathbb{N}$, with $a < b$, we denote with $\dsb{a,b}\coloneqq \{a,\dots,b\}$ and $\dsb{a} \coloneqq \dsb{1,a}$.} where $T \in \mathbb{N}$ is the learning horizon, the learner selects an arm ${I_t} \in \dsb{K}$ among a finite set of $K$ arms and observes a realization of the reward $X_{I_t, t}$. The reward for each arm $i \in \dsb{K}\coloneqq \{1, \ldots, K\}$ at round $t \in \dsb{T}$ is modeled by a random variable $X_{i, t}$ described by a distribution unknown to the learner. We denote by $\mu_{i,t} \coloneqq \mathbb{E}[X_{i,t}]$ the corresponding expected reward. We study two types of distributions of the rewards encoded by the following assumptions.

\begin{ass}[Bernoulli rewards]\label{ass:bernoulli}
For every arm $i \in \dsb{K}$ and round $t \in \dsb{T}$, the reward $X_{i,t}$ is s.t.~$X_{i,t}\sim \textit{Be}(\mu_{i,t})$, where $\textit{Be}(\mu)$ denotes a Bernoulli distribution with parameter $\mu \in [0, 1]$.
\end{ass}

\begin{ass}[Subgaussian rewards]\label{ass:subg}
For every arm $i \in \dsb{K}$ and round $t \in \dsb{T}$, the reward $X_{i,t}$ is s.t.~$X_{i,t} \sim \textit{SubG}(\mu_{i,t}, \lambda^2)$, where $\textit{SubG}(\mu,\lambda^2)$ denotes a generic subgaussian distribution with finite mean $\mu\in \mathbb{R}$ and proxy variance $\lambda^2$.\footnote{A random variable $X$ with expectation $\mu$ is $\lambda^2$-subgaussian if for every $s \in \mathbb{R}$ it holds that $\E[\exp(s(X-\mu))] \le \exp(s^2\lambda^2/2).$}
\end{ass}

The goal of the learner $\mathfrak{A}$ is to minimize the \emph{expected cumulative dynamic frequentist regret} $R_T(\mathfrak{A})$ over the learning horizon $T$, defined as the cumulative difference between the reward of an oracle that chooses at each time the arm with the largest expected reward at round $t$, defined as $i^*(t) \in \mathop{\text{argmax}}_{i \in \dsb{K}} \mu_{i,t}$, and expected reward $\mu_{I_t,t}$ of the arm $I_t$ selected by the learner for the round, formally:
\begin{equation}
    R_T(\mathfrak{A}) : = \mathbb{E}\left[\sum_{t=1}^T \left(\mu_{i^*(t),t} - \mu_{I_t,t} \right)\right],
\end{equation}
where the expected value is taken w.r.t.~the randomness of the rewards and the possible randomness of the algorithm. In the following, as is often done in the NS-MABs literature (e.g., \cite{besson2019generalized,liu2018change, re2021exploiting, trovo2020sliding, garivier2011upper}) we provide results on the expected value of the pull of the arms $\mathbb{E}[N_{i,T}]$, where $N_{i,T}$ is the random variable representing the number of total pulls of the arm $i$ at round $T$ excluding the rounds in which $i$ is optimal, formally defined as $N_{i,T}=\sum_{t=1}^T\mathds{1}\{I_t=i,\, i\neq i^*(t)\}$.

\section{Algorithms}\label{sec:alg}
We analyze two \emph{sliding-window} algorithms, namely the \texttt{Beta-SWTS}, proposed in~\cite{trovo2020sliding}, and the \texttt{$\gamma$-SWGTS}, introduced by~\citet{fiandri2025thompsonsamplinglikealgorithmsstochastic}, both inspired by the classical TS algorithm. Similarly to what happens with \texttt{SW-UCB}, they handle the problem posed by the dynamical nature of the expected rewards by exploiting only the subset of the most recent collected rewards, i.e., within a {sliding window} of size $\tau \in \mathbb{N}$. This allows us to handle the bias given by the least recent collected rewards, which, in an NS-MAB, may be non-representative of the current expected rewards.

The pseudocode of \texttt{Beta-SWTS} for Bernoulli-distributed rewards is presented in Algorithm~\ref{alg:swbetats}, while the pseudocode of \texttt{$\gamma$-SWGTS} for subgaussian rewards is presented in Algorithm~\ref{alg:swgts}. They are based on the principle of \emph{conjugate-prior} updates.
The key difference from the classical TS stands in discarding older examples, thanks to the window width $\tau$, through a sliding-window mechanism. This way, the prior remains sufficiently spread over time, ensuring ongoing exploration, essential to deal with non-stationarity.

For every round $t \in \dsb{T}$ and arm $i \in \dsb{K}$, we denote with $\nu_{i,t}$ the prior distribution for the parameter $\mu_{i,t}$ after $t$ rounds. 
For \texttt{Beta-SWTS}, an uninformative prior is set, i.e., $\nu_{i,1} \coloneqq Beta(1, 1)$ (Line \ref{line:beta2}), where $Beta(\alpha, \beta)$ is a Beta distribution with parameters $\alpha,\beta \ge 0$. The posterior of the expected reward of arm $i$ at round $t$ is given by $\nu_{i,t} \coloneqq Beta(S_{i,t,\tau} +
1, N_{i,t,\tau} - S_{i,t,\tau} + 1)$, where $N_{i,t,\tau} \coloneqq \sum_{
s = \max{\{t-\tau,1}\}}^{t-1} \mathds{1}{\{I_s = i\}}$ is the number of times arm $i$ was selected in the last $\min{ \{t, \tau \}}$ rounds, and $S_{i,t,\tau} \coloneqq \sum_{ s = \max\{{t-\tau,1}\}}^{t-1} X_{i,s} \mathds{1}{\{I_s = i\}}$ is the cumulative reward collected by arm $i$ in the last $\min{\{t, \tau}\}$ rounds. At each round $t$ and for each arm $i$, the algorithm draws a random sample from $\theta_{i,t,\tau}$, a.k.a.~Thompson sample (Line \ref{line:sample2}); then, the arm whose sample has the largest value gets played (Line \ref{line:selectionts2}). Based on the collected reward $X_{I_t,t}$ the prior distributions $\nu_{i,t+1}$ are updated  (Line \ref{line:updatets2}).
\texttt{$\gamma$-SWGTS} algorithm shares the same principles of \texttt{Beta-SWTS} with some differences. In particular, after $K$ rounds of initialization in which every arm is played once (Line \ref{line:roundRob}), at every round $t$, the prior distribution is defined as $\nu_{i,t} \coloneqq \mathcal{N}\left(\frac{S_{i,t,\tau}}{N_{i,t,\tau}},\frac{1}{\gamma N_{i,t,\tau}}\right)$, where $\mathcal{N}(\alpha, \beta)$ is a Gaussian distribution with mean $\alpha \in \mathbb{R}$ and variance $\beta \ge 0$, with $S_{i,t,\tau}$ and $N_{i,t,\tau}$ defined as above, and $\gamma > 0$ is a hyperparameter whose value will be set later. At each round $t$ and for each arm $i$, the algorithm draws a random sample $\theta_{i,t,\tau}$ from $\nu_{i,t}$ (Line \ref{line:sample11}) and the arm with the largest Thompson sample is played (Line \ref{line:selectionts11}).  Whenever there is no information about an arm, i.e., when $N_{i, t, \tau} = 0$, the arm is forced to play, so that the prior distribution is always well defined (Line \ref{line:noInfo}).  Then, based on the collected reward $X_{I_t,t}$ the prior distributions $\nu_{i,t+1}$ are updated  (Line \ref{line:updSWGT}).

{\small
\begin{algorithm}[th!]
\caption{\texttt{Beta-SWTS}} \label{alg:swbetats}
\begin{algorithmic}[1]
    \STATE \textbf{Input:} Number of arms $K$, learning horizon $T$, window $\tau$
    \STATE Set $S_{i,1,\tau} \gets 0$ for each $i \in \dsb{K}$
    \STATE Set $\nu_{i,1} \gets Beta(1, 1)$ for each $i \in \dsb{K}$ \label{line:beta2}
    \FOR{$t \in \dsb{T}$}
        \STATE Sample $\theta_{i,t,\tau} \sim \nu_{i,t}$ for each $i \in \dsb{K}$ \label{line:sample2}
        \STATE Select $I_t \in \arg \max_{i \in \dsb{K}} \theta_{i,t,\tau}$ \label{line:selectionts2}
        \STATE Pull arm $I_t$
        \STATE Collect reward $X_{I_t,t}$
        \STATE Update $S_{i,t+1,\tau}$ and $N_{i,t+1,\tau}$ for each $i \in \dsb{K}$
        \STATE Update $\nu_{i,t+1} \gets Beta(1+S_{i,t+1,\tau},1+(N_{i,t+1,\tau}-S_{i,t+1,\tau}))$ for each $i \in \dsb{K}$ \label{line:updatets2}
    \ENDFOR
\end{algorithmic}
\end{algorithm}
}
{\small
\begin{algorithm}[th!]
\caption{\texttt{$\gamma$-SWGTS}} \label{alg:swgts}
\begin{algorithmic}[1]
    \STATE \textbf{Input:} Number of arms $K$, learning horizon $T$, parameter $\gamma$, window $\tau$
    \STATE Play every arm once: 
    \FOR{$t \in \dsb{K}$}\label{line:roundRob}
        \STATE Pull arm $I_t=t$
        \STATE Collect reward $X_{I_t,t}$
        \STATE Set $S_{I_t,K+1,\tau}\gets X_{I_t,t}$
    \ENDFOR
    \STATE Set $\nu_{i,K+1} \gets \mathcal{N}(S_{i,K+1,\tau},\frac{1}{\gamma})$ for each $i \in \dsb{K}$ \label{line:prior1}
    \FOR{$t \in \dsb{K+1,T}$}
         \IF{ $\exists i \in \dsb{K} \quad s.t. \quad N_{i,t,\tau}=0$}\label{line:noInfo}
            \STATE Select $I_t=i$
        \ELSE
            \STATE Sample $\theta_{i,t, \tau} \sim \nu_{i,t}$ for each $i \in \dsb{K}$ \label{line:sample11}
            \STATE Select $I_t \in \arg \max_{i \in \dsb{K}} \theta_{i,t,\tau}$ \label{line:selectionts11}
        \ENDIF
        \STATE Pull arm $I_t$
            \STATE Collect reward $X_{I_t,t}$
            \STATE Update  $S_{i,t+1,\tau}$ and $N_{i,t+1,\tau}$ for each $i \in \dsb{K}$ 
            \STATE Update $\nu_{i,t+1} \gets \mathcal{N}\label{line:updatets11}\left( \frac{S_{i,t+1,\tau}}{N_{i,t+1,\tau}},\frac{1}{\gamma N_{i,t+1,\tau}} \right) $ for each $i \in \dsb{K} $ \label{line:updSWGT}
    \ENDFOR
\end{algorithmic}
\end{algorithm}
}

\section{Regret Analysis for the General Non-Stationary Environment}\label{sec:gg}
In this paper, we investigate NS-MABs in a unifying framework allowing the mean rewards $\mu_{i,t}$ to change arbitrarily over time with no particular regularity, as long as the Assumption \ref{ass:bernoulli} or Assumption \ref{ass:subg} is met. Beginning from this general regret analysis, in Sections \ref{sec:ac} and \ref{sec:sc}, we particularize it for the cases in which $\mu_{i,t}$ satisfies additional regularity conditions, i.e., abrupt and smoothly changing, respectively.

We start the analysis by introducing a definition to characterize the rounds during which the algorithms can effectively assess the best arm even in the presence of non-stationarity.

\begin{definition}[Unlearnable set $\mathcal{F}_{\tau}$ and learnable set $\mathcal{F}_{\tau}^{\complement}$]\label{def:sets}
For every window size $\tau \in \mathbb{N}$, the \emph{unlearnable set} $\mathcal{F}_{\tau}$ is defined as any superset of $\mathcal{F}_{\tau}'$ defined as:
\begin{equation}
    \mathcal{F}_{\tau}' \coloneqq \left\{t \in \dsb{T} : \exists i \in \dsb{K}\setminus i^*(t), \min_{t' \in \dsb{t-\tau, t-1}}\{\mu_{i^*(t),t'}\} \leq \max_{t' \in \dsb{t-\tau, t-1}}\{ \mu_{i,t'}\} \right\},
\end{equation}
and the \emph{learnable set} $\mathcal{F}_{\tau}^{\complement}$ is defined as $ \mathcal{F}_{\tau}^{\complement} \coloneqq \dsb{T} \setminus \mathcal{F}_{\tau}$.
\end{definition}

Notice that by definition, for every round $t \in \mathcal{F}_{\tau}^{\complement}$, the following inequality holds true for all $ i \neq i^*(t)$:
$$
\min_{t' \in \dsb{t-\tau,t-1}}\{\mu_{i^*(t),t'}\} > \max_{t' \in \dsb{t-\tau,t-1}}\{ \mu_{i,t'}\}.
$$
Intuitively, $\mathcal{F}_{\tau}^{\complement}$ collects all the rounds $t \in \dsb{T}$ such that the smallest expected reward of the optimal arm $i^*(t)$ within the last $\tau$ rounds is larger than the largest expected reward of all other arms in the same interval spanning the length of the sliding window $\tau$. This enables the introduction of a general definition for the suboptimality gaps $\Delta_\tau$ that encodes how challenging it is to identify the optimal arm relying on the rewards collected in the past $\tau$ rounds only. Formally:

\begin{definition}[Generalized sub-optimality gap $\Delta_{\tau}$] \label{def:subgap}
For every window size $\tau \in \mathbb{N}$, the general suboptimality gap is defined as follows:
\begin{equation}
    \Delta_{\tau} := \min_{t \in \mathcal{F}_{\tau}^{\complement}, i \in \dsb{K}\setminus i^*(t)}\left\{\min_{t' \in \dsb{t-\tau,t-1}}\{\mu_{i^*(t),t'}\} - \max_{t' \in \dsb{t-\tau,t-1}}\{ \mu_{i,t'}\} \right\}.
\end{equation}  
\end{definition}
{
The suboptimality gap $\Delta_{\tau}>0$ quantifies a minimum non-zero distance in terms of expected reward between the optimal arm $i^*(t)$ and all other arms across all rounds $t \in \mathcal{F}_{\tau}^{\complement}$.}
We are now ready to present the result on the upper bound of the expected number of pulls for the analyzed algorithms.

\begin{restatable}[General Analysis for \texttt{Beta-SWTS}]{theorem}{swbetagen} \label{thr:gen1}
Under Assumption~\ref{ass:bernoulli} and $\tau \in \mathbb{N}$, for \texttt{Beta-SWTS} the following holds true for every arm $i \in \dsb{K}$:
   \begin{align}
      \mathbb{E}[N_{i,T}] \leq O\left(\textcolor{vibrantRed}{\underbrace{|\mathcal{F}_{\tau}|}_{(\textnormal{A})}}+\textcolor{vibrantBlue}{\underbrace{\frac{T\ln(\tau)}{\Delta_\tau^2\tau}}_{(\textnormal{B})}}\right).
   \end{align}
\end{restatable}

\begin{restatable}[General Analysis for \texttt{$\gamma$-SWGTS}]{theorem}{swgaussgen}\label{thr:gen2}
Under Assumption \ref{ass:subg}, $\tau \in \mathbb{N}$, for \texttt{$\gamma$-SWGTS} with $\gamma\leq \min\{\frac{1}{4\lambda^2},1\}$ the following holds true for every arm $i \in \dsb{K}$:
   \begin{align}
      \mathbb{E}[N_{i,T}] \leq O\left(\textcolor{vibrantRed}{\underbrace{|\mathcal{F}_{\tau}|}_{(\textnormal{A})}}+\textcolor{vibrantBlue}{\underbrace{\frac{T\ln(\tau\Delta_{\tau}^2+e^6)}{\gamma\Delta_\tau^2\tau}}_{(\textnormal{B})}}+\textcolor{vibrantTeal}{\underbrace{\frac{T}{\tau}}_{(\textnormal{C})}}\right).
   \end{align}
\end{restatable}

These results capture a trade-off in the choice of the window size $\tau$.
Specifically, we observe that, given a window size $\tau$, the regret is decomposed in two contributions, namely: \textcolor{vibrantRed}{$(\text{A})$}, being the \textcolor{vibrantRed}{\emph{the cardinality of the unlearnable set}} $|\mathcal{F}_\tau|$, i.e., a superset of the set of rounds in which no algorithm exploiting only the $\tau$ most recent samples can distinguish consistently the best arm from the suboptimal ones; \textcolor{vibrantBlue}{$(\text{B})$}, corresponding \textcolor{vibrantBlue}{\emph{the expected number of pulls of the suboptimal arm within the the learnable set}}. We can see that \textcolor{vibrantRed}{$(\text{A})$}$=|\mathcal{F}_\tau|$ tends to increase with $\tau$ and \textcolor{vibrantBlue}{$(\text{B})$} decreases with $\tau$.
Notice that dealing with subgaussian reward, a term that accounts for the (possibly) greater uncertainty for the realization of the rewards appears, namely $\gamma$. Similarly, an additional \textcolor{vibrantTeal}{\text{(C)}} term arises for \texttt{$\gamma$-SWGTS}, taking into account the forced exploration to ensure the posterior distribution is always well defined. In the next sections, we discuss how these results compare to the ones retrieved in the literature for the most common stationary bandits.

Figure~\ref{fig:Remark} provides an example showing how the choice of the window size $\tau$ affects the cardinalities of $\mathcal{F}_\tau$ and $\mathcal{F}_{\tau}^\complement$. The figure depicts a setting in which the optimal arm is the same until an abrupt change occurs. This partitions the learning horizon into the $\mathcal{I}_1$, ${\mathcal{I}_2}$, and ${\mathcal{I}_3}$ intervals. We consider three different values for the window size $\tau_1 > \tau_2 > \tau_3$. As the window size increases, the cardinality of $\mathcal{F}^\complement_{\tau}$ decreases, as depicted below the figure. Indeed, the learnable sets exclude those rounds for which the window overlaps with two different intervals. Conversely, when we set a small window, e.g., $\tau_3$, the set $\mathcal{F}_{\tau_3}^\complement$ includes more rounds while guaranteeing that a generic algorithm exploiting samples from the window is capable of selecting the best arm consistently. This is due to the fact that, for smaller window size, the algorithms are able to adapt faster to the new form of the expected rewards.
However, choosing $\tau$ too small, as suggested by term \textcolor{vibrantBlue}{$(\text{B})$} of Theorems \ref{thr:gen1} and \ref{thr:gen2}, can lead to a large number of pulls of the suboptimal arms, proportional to $\widetilde{O}\left(\frac{T}{\tau}\right)$, as the algorithms become too explorative.

As a final remark, we highlight that we do not ask for any specific regularity for the expected rewards, so the results hold for any arbitrary NS-MAB, e.g., also for the rising restless \citep{metelli2022stochastic} or the rotting restless bandits \citep{pmlr-v108-seznec20a}. Now, we are ready to show the results these theorems imply for the most common NS-MAB, i.e., abruptly changing and smoothly changing ones.

\begin{figure}[t]
    \centering
    \includegraphics[]{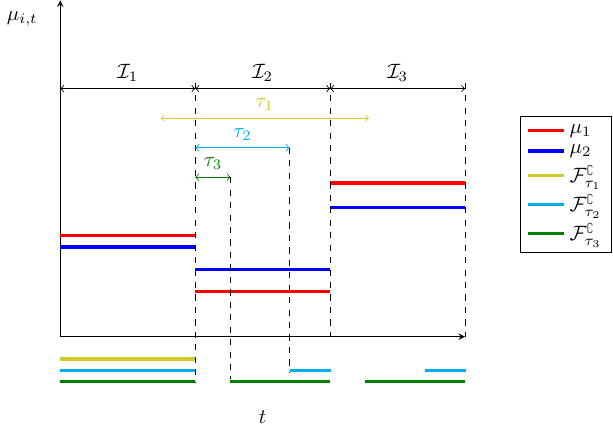}
    \caption{Piecewise-constant abruptly-changing bandit setting, showing arms' expected reward (red and blue), phases, different window sizes, and learnable sets (yellow, light blue and green).}
    \label{fig:Remark}
\end{figure}

\section{Regret Analysis for Abruptly Changing Environments}\label{sec:ac}

{We now consider the \emph{piece-wise constant abruptly-changing} environment, i.e., those scenarios in which the expected rewards of the arms remain the same during subsets of the learning horizon called phases, and the phase changes at unknown rounds called breakpoints (Figure \ref{fig:esempio}). First, we introduce some quantities used to characterize the regret. Second, we express Theorem \ref{thr:gen1} and Theorem \ref{thr:gen2} in terms of these newly defined quantities, comparing them with those of the state-of-the-art algorithms devised for this setting. Finally, we show that our results apply to a far more general class of \emph{abruptly-changing} NS-MABs where the expected reward is not constrained to remain constant within each phase.

\begin{definition}[Breakpoint]\label{def:break}
    A breakpoin is a round $t \in \dsb{2,T}$ such that there exists $i \in \dsb{K}$ for which holds $\mu_{i,t}\neq\mu_{i,t-1}$  
\end{definition}
}

Let us denote with $b_\psi$ as the $\psi$-th breakpoint $1 < b_1 < \ldots < b_{\Upsilon_T} < T$, where $\Upsilon_T \in \dsb{T}$ is the total number of breakpoints over a learning horizon $T$.
The breakpoints partition the learning horizon $\dsb{T}$ into phases $\mathcal{F}_\psi$ and pseudophases $\mathcal{F}_{\psi,\tau}^*$. Formally, using the convention that $b_0 = 1$ and $b_{\Upsilon_T+1} = T$:
\begin{definition}[Phase $\mathcal{F}_\psi$]\label{def:phase}
    Let $T \in \mathbb{N}$ be the learning horizon and $\psi \in \dsb{\Upsilon_T+1}$, we define the $\psi$-th phase as:
    \begin{equation}
        \mathcal{F}_{\psi} \coloneqq \{ t \in \dsb{T} \,:\, t \in \dsb{b_{\psi-1}, b_\psi-1} \}.
    \end{equation}
\end{definition}
It is worth noting that the optimal arm $i^*(t)$ is for sure constant within each phase $\psi \in \dsb{\Psi_T+1}$, i.e., we can appropriately denote it as $i^*_\psi$.
\begin{definition}[Pseudophase, $\mathcal{F}_{\psi,\tau}^*$]\label{def:pseudophase}
   Let $T \in \mathbb{N}$ be the learning horizon, a window size $\tau$, and $\psi \in \dsb{2,\Upsilon_T+1}$, the $\psi$-th pseudophase is defined as:
   \begin{equation}
    \mathcal{F}_{\psi, \tau}^* \coloneqq \{t \in \dsb{T} \,:\, t \in \dsb{b_{\psi-1} + \tau, b_\psi-1} \},
    \end{equation}
and $\mathcal{F}_{1,\tau}^* = \mathcal{F}_1$.\footnote{When $\tau$ is longer than the phase, the pseudophase is empty, i.e., where $\mathcal{F}_{\psi,\tau}^* = \{\}$ for $\tau \ge b_{\psi} - b_{\psi-1}$.} 
\end{definition}
Finally, we define $\mathcal{F}_\tau^*= \bigcup_{\psi= 1}^{\Upsilon_T+1} \mathcal{F}_{\psi,\tau}^*$.
The intuition behind the definition of the pseudophase is that if we use an algorithm $\mathfrak{A}$ relying on a sliding window of size $\tau$ during the rounds of the pseudophase $\mathcal{F}^*_{\psi,\tau}$, the algorithm $\mathfrak{A}$ uses only on rewards belonging to the single phase $\mathcal{F}_\psi$. We provide a graphical representation of the definitions introduced above in Figure~\ref{fig:esempio}. In particular, we have two breakpoints ($\Upsilon_T = 2$), and three phases $\mathcal{F}_1$, $\mathcal{F}_2$, and $\mathcal{F}_3$. Given a window size of $\tau$, we have three pseudophases $\mathcal{F}_{1,\tau}^*$, $\mathcal{F}_{2,\tau}^*$, and $\mathcal{F}_{3,\tau}^*$, where the last two pseudophases start $\tau$ rounds after the start of the corresponding phase.

Let us characterize the sets introduced in Definition~\ref{def:sets}, namely $\mathcal{F_\tau}$ and $\mathcal{F_\tau}^\complement$, using the concepts of phase and pseudophase. We can express $\mathcal{F}_{\tau}$ as the union of the set of rounds of length $\tau$ after every breakpoint, formally:
$$
\mathcal{F}_{\tau}=\bigcup_{\psi \in \dsb{\Upsilon_T+1}} \mathcal{F}_{\psi} \setminus \mathcal{F}_{\psi,\tau}^*.
$$
Consequently, we have $\mathcal{F}_{\tau}^{\complement}=\mathcal{F}_\tau^*$. Therefore, since for any round $t \in \dsb{T}$ belonging to a pseudophase, the algorithms using a sliding window of size $\tau$ uses samples coming from a single phase, we have that for any $t \in \mathcal{F}_\tau^*$:
$$
\min_{t'\in \dsb{t-\tau,t-1}}\{\mu_{i^*(t),t'}\}>\max_{t' \in \dsb{t-\tau,t-1}, i \in \dsb{K}\setminus\{i^*(t)\}}\{\mu_{i,t'}\},
$$ 
which corresponds to the learnable set in Definition \ref{def:sets}. The latter inequality follows from the fact that any round $t \in \mathcal{F}_\tau^*$ belongs to a pseudophase $\mathcal{F}_{\psi,\tau}^*$ and, therefore, all the times $t' \in \dsb{t-\tau,t-1}$ belong to a single phase $\mathcal{F}_{\psi}$. 
By definition of the general suboptimality gap (Definition \ref{def:subgap}), we have:
\begin{align}
    \Delta_{\tau}=\min_{t \in \mathcal{F}_\tau^*, i \in \dsb{K}\setminus i^*(t)}\left\{\min_{t' \in \dsb{t-\tau,t-1}}\{\mu_{i^*(t),t'}\} - \max_{t' \in \dsb{t-\tau,t-1}}\{ \mu_{i,t'}\}\right\}.
\end{align}
Notice that the definition of $\Delta_{\tau}$, if $\tau$ is such that no pseudophase is empty, corresponds to the definition of $\Delta$ in the work by~\cite{garivier2008upper} in the case of piecewise-constant setting.

We are now ready to present the results on the upper bounds of the number of plays in the abruptly changing environment.
\begin{restatable}[Analysis for \texttt{Beta-SWTS} for for Piece-Wise Constant Abruptly Changing Environments]{theorem}{restlessbeta}\label{thr:restlessbetamain}
    Under Assumptions \ref{ass:bernoulli}, $\tau \in \mathbb{N}$, for \texttt{Beta-SWTS} the following holds:
        \begin{align}
         \mathbb{E}[N_{i,T}] \leq O\left(\textcolor{vibrantRed}{\Upsilon_T\tau}+\textcolor{vibrantBlue}{\frac{T\ln(\tau)}{\Delta_\tau^2\tau}}\right).
    \end{align}
\end{restatable}

\begin{restatable}[Analysis for \texttt{$\gamma$-SWGTS} for Piece-Wise Constant Abruptly Changing Environments]{theorem}{restlessgauss}\label{thr:restlessgaussmain}
Under Assumptions \ref{ass:subg}, $\tau \in \mathbb{N}$, for \texttt{$\gamma$-SWGTS}  with $\gamma \leq \min\{\frac{1}{4\lambda^2},1\}$ it holds that:
    \begin{align}
         \mathbb{E}[N_{i,T}] \leq O \left(\textcolor{vibrantRed}{\Upsilon_T\tau}+\textcolor{vibrantBlue}{\frac{T\ln(\tau\Delta_{\tau}^2+e^6)}{\gamma\Delta_\tau^2\tau}}+\textcolor{vibrantTeal}{\frac{T}{\tau}}\right).
    \end{align}
\end{restatable}

Let us further analyze the bounds obtained. Making a direct comparison with Theorem \ref{thr:gen1} and \ref{thr:gen2} for the general NS-MAB setting, we now appreciate a clearer formulation for \textcolor{vibrantRed}{\emph{the cardinality of the unlearnable set}}. In fact, in abruptly changing environments, is convenient to characterize the unlearnable set as the set of rounds length $\tau$ after every breakpoint. In these $\textcolor{vibrantRed}{\Upsilon_T\tau}$ rounds, we cannot guarantee that the algorithms will be able to distinguish the best arm from the suboptimal ones. Figure \ref{fig:esempio} provides an explicit graphical representation of the quantities introduced. In particular, we see that in the first $\tau$ rounds of each phase, the rewards gathered within the window size are not representative of the current expected rewards, as they may include examples from rounds in which the ranking of the arms is different. The order for \textcolor{vibrantBlue}{\emph{the expected number of pulls of the suboptimal arm within the the learnable set}} matches the state-of-the-art order in $T$, $\tau$, and $\Delta_\tau$ for the expected number of pulls for a sliding window algorithm, even when applied to a stationary bandit \citep{garivier2008upper}. 

Since existing algorithms for this setting are devised to handle environments with expected rewards bounded in $[0,1]$, in order to compare the results obtained we only consider the piecewise-constant abruptly-changing environment with Bernoulli rewards.
 Let us assume $\Delta_\tau$ constant w.r.t.~$T$, as done in the NS-MAB literature \citep{garivier2008upper,liu2018change,besson2019generalized,re2021exploiting} and let us choose $\tau \propto \sqrt{\frac{T\ln(T)}{\Upsilon_T}}$. From Theorem~\ref{thr:restlessbetamain} and~\ref{thr:restlessgaussmain}, we derive the following guarantees on the regret:\footnote{Here, we also neglect the dependence on $\gamma$ for \texttt{$\gamma$-SWGTS}.}
\begin{equation}
    R_T(\texttt{Beta-SWTS}/\texttt{$\gamma$-SWGTS})\leq O \left(\frac{1}{\Delta_{\tau}^2}{\sqrt{\Upsilon_T T\ln(T)}}\right),
\end{equation}
that is the same order of the guarantees on the regret of \texttt{SW-UCB} \citep[][Theorem $7$]{garivier2008upper}.
Even if \texttt{GLR-klUCB} relies on an active approach to deal with non-stationary bandits, it also retrieves the same order for the bounds on the regret \citep[][Theorem 5]{besson2019generalized}. 
Finally, \texttt{CUSUM-UCB} and \texttt{BR-MAB} can achieve the following upper bound on the regret \citep[][Corollary 2, Theorem 4]{liu2018change,re2021exploiting}:
\begin{equation}\label{remar:comparison}
       R_T(\texttt{CUSUM-UCB}/\texttt{BR-MAB})\leq O \left(\frac{1}{\Delta_\tau^2}\sqrt{\Upsilon_T T \ln\left({\frac{T}{\Upsilon_T}}\right)}\right),
   \end{equation}
which is better than the previous one only for a $\Upsilon_T$ factor in the logarithmic term.

The results of Theorem \ref{thr:gen1} and Theorem \ref{thr:gen2} hold for a way more general setting than the piece-wise constant abruptly-changing NS-MABs. In Figure \ref{fig:esempio2}, we highlight the rounds belonging to the unlearnable set in yellow and the rounds belonging to the learnable set in green for a setting in which the expected rewards \emph{are not constant} but the expected reward of the optimal arm never intersects that of the suboptimal ones in every phase. Note that the cardinality of the learnable and unlearnable sets are the same as those of the NS-MAB described by Figure \ref{fig:esempio}. Thus, it is not surprising that Theorem \ref{thr:restlessbetamain} and Theorem \ref{thr:restlessgaussmain} hold even for the second setting. This represents a generality of our analysis that, to the best of the authors' knowledge, is not captured by the existing NS-MAB literature. We refer to the class of NS-MABs as \emph{(general) abruptly-changing}, which can be formally defined through a notion of \emph{general breakpoint}.
\begin{definition}[General Breakpoints]\label{def:genbreak}
   A set of $\Upsilon_T+1$ rounds $1 \eqqcolon b_0  < b_1 < \dots <b_{\Upsilon_T} < T \coloneqq b_{\Upsilon_T+1}$ are generalized breakpoints if for every $\psi \in \dsb{\Upsilon_T+1}$ it holds that:
   \begin{align}
        \min_{t\in\dsb{b_{\psi-1},b_\psi-1}}\{\mu_{i^*(t),t}\}>\max_{t\in\dsb{b_{\psi-1},b_\psi-1}}\{\mu_{i,t}\},
    \end{align}
    for every arm $i  \in \dsb{K}\setminus \{i^*(t)\}$.
\end{definition}

Notice that, similarly to the previous case, by definition, the optimal arm does not change within two breakpoints, i.e., $ i^*(t) = i^*_\psi$ for every $t \in \dsb{b_{\psi-1},b_\psi-1}$ and interval $\psi \in \dsb{\Upsilon_T+1}$.
The definitions of phases and pseudophases (Definition \ref{def:phase} and Definition \ref{def:pseudophase}) still hold with the new definition of the breakpoint. {Again, when sampling within an arbitrary pseudophase $\mathcal{F_{\psi,\tau}^*}$, since we use only samples belonging to phase $\mathcal{F_\psi}$ for which it holds by definition that $\min_{t\in\dsb{b_{\psi-1},b_\psi-1}}\{\mu_{i^*(t),t}\}>\max_{t\in\dsb{b_{\psi-1},b_\psi-1}}\{\mu_{i,t}\}$, also the following holds true or any $t \in \mathcal{F}_\tau^*$ (recalling that $\mathcal{F}_\tau^*=\bigcup_{\psi\in \dsb{\Upsilon+1}}\mathcal{F}_{\psi,\tau}^*$):
$$
\min_{t'\in \dsb{t-\tau,t-1}}\{\mu_{i^*(t),t'}\}>\max_{t' \in \dsb{t-\tau,t-1}, i \in \dsb{K}\setminus\{i^*(t)\}}\{\mu_{i,t'}\},
$$ 
which corresponds to the learnable set in Definition \ref{def:sets}.}
\begin{figure}[t]
\centering
    \subfloat[]{
        \scalebox{1}{\includegraphics[]{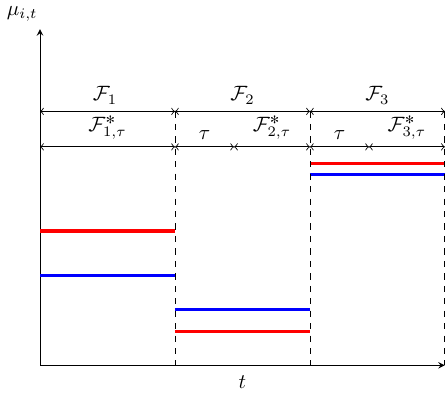}}
        \label{fig:esempio}
    }
    \subfloat[]{
        \scalebox{1}{\includegraphics[]{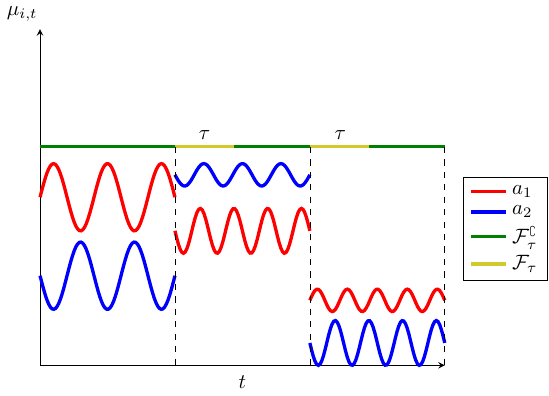}}
        \label{fig:esempio2}
    }
    \caption{Two abruptly changing environments: (a) the classical piecewise-constant environment, (b) the general abruptly changing. The figures also provide a depiction of phases $\mathcal{F}_i$ and pseudophase $\mathcal{F}^*_i$.}
    \label{fig:abruptlyintuition}
\end{figure}    

\section{Regret Analysis for Smoothly Changing Environments}\label{sec:sc}

We now study what can be inferred from Theorems \ref{thr:gen1} and \ref{thr:gen2} in the \emph{smoothly changing} environments, i.e., those scenarios in which the expected reward of each arm is allowed to vary only for a limited amount between consecutive rounds. { The regret analysis through breakpoints is unsuitable for an environment in which the expected rewards evolve smoothly. In what follows, we characterize the regret the algorithms suffer in these settings introducing the most common definitions and assumptions used in the smoothly changing environment literature, deriving the implications for the sets introduced in Definition \ref{def:sets}. Finally, we compare our results with the state-of-the-art results for the setting.}
\begin{ass}[Lipschitz continuity, \cite{combes2014unimodal,trovo2020sliding}]\label{ass:smooth}
The expected reward of the arms is Lipschitz continuous if there exists $\sigma < +\infty$ such that for every round $t, t'\in \dsb{T}$ and arm $i \in \dsb{K}$ we have:
\begin{equation}
    |\mu_{i,t} - \mu_{i,t'}| \leq \sigma |t - t'|.    
\end{equation}
\end{ass}

\begin{ass}[Smoothness, \cite{combes2014unimodal, trovo2020sliding}]\label{ass:smooth2}
Let $\Delta' > 2 \sigma \tau > 0$ be finite, we define $\mathcal{F}_{\Delta',T}$ as:
\begin{align}
    \mathcal{F}_{\Delta',T}\coloneqq \left\{t \in \dsb{T} \,:\, \exists i,j \in \dsb{K}, i\neq j, |\mu_{i,t}-\mu_{j,t}| < \Delta'\right\}.
\end{align}
There exist $\beta \in [0,1]$ and finite $F <+\infty$, such that $|\mathcal{F}_{\Delta',T}|\leq FT^{\beta}$.
\end{ass}

Notice that Assumption $1$ in \cite{combes2014unimodal} is a particular case of the above assumption when $\beta = 1$. We, instead, follow the line of \cite{trovo2020sliding}, considering an arbitrary order of $T^\beta$.
In the proof of Theorem \ref{thr:betasmooth}, we show that, under Assumptions \ref{ass:smooth} and \ref{ass:smooth2}, considering the complement set $\mathcal{F}_{\Delta',T}^{\complement}\coloneqq \dsb{T}\setminus \mathcal{F}_{\Delta',T}$, for every round $t \in \mathcal{F}_{\Delta',T}^{\complement}$, it holds that:
\begin{equation}
\min_{t' \in \dsb{t-\tau,t-1}}\{\mu_{i^*(t),t'}\} - \max_{t' \in \dsb{t-\tau,t-1}}\{ \mu_{i,t'}\} \ge \Delta^\prime - 2\sigma\tau > 0,
\end{equation}
This implies that $\mathcal{F}_{\tau}= \mathcal{F}_{\Delta',T}$. From this fact, it is easy to prove that also $\Delta_{\tau} = \Delta'-2\sigma\tau$.

We are now ready to present the results on the upper bounds of the number of pulls of suboptimal arms for the smoothly changing environment.
\begin{restatable}[Analysis for \texttt{Beta-SWTS} for Smoothly Changing Environments]{theorem}{swbetasmooth}\label{thr:betasmooth}
Under Assumptions  \ref{ass:bernoulli}, \ref{ass:smooth}, and \ref{ass:smooth2}, $\tau \in \mathbb{N}$, for \texttt{Beta-SWTS}, it holds that:
    \begin{equation}\label{eq:rob2}
      \mathbb{E}[N_{i,T}]\leq O \left( \textcolor{vibrantRed}{FT^{\beta}} +\textcolor{vibrantBlue}{\frac{T\ln(\tau)}{(\Delta'-2\sigma\tau)^2\tau}}\right).
    \end{equation}
\end{restatable}
\begin{restatable}[Analysis for \texttt{$\gamma$-SWGTS} for Smoothly Changing Environments]{theorem}{swgtssmooth}\label{thr:gausssmooth}
Under Assumptions \ref{ass:subg}, \ref{ass:smooth}, and \ref{ass:smooth2}, $\tau \in \mathbb{N}$, for \texttt{$\gamma$-SWGTS} with $\gamma\leq \min \left\{ \frac{1}{4 \lambda^2}, 1 \right\}$, it holds that:
\begin{equation}\label{eq:rob}
  \mathbb{E}[N_{i,T}]\leq O \left( \textcolor{vibrantRed}{FT^{\beta} }+\textcolor{vibrantBlue}{\frac{T\ln(\tau(\Delta'-2\sigma\tau)^2+e^6)}{\gamma(\Delta'-2\sigma\tau)^2\tau}}+\textcolor{vibrantTeal}{\frac{T}{\tau}}\right).
  \end{equation}
\end{restatable}
{Again, we identify the two main contributions, \emph{\textcolor{vibrantRed}{the cardinality of the unlearnable set}} and  \emph{\textcolor{vibrantBlue}{the expected number of pulls within the learnable set}}. The former can be bounded, under Assumption \ref{ass:smooth2}, by $FT^\beta$
The latter is characterized by a sub-optimality gap $\Delta_\tau$ that depends on the smoothness parameter $\sigma$ and on the window size $\tau$, capturing the fact that in the rounds in which the distance between the best arm and the suboptimal ones is lower-bounded by $\Delta'$ (as defined in Assumption \ref{ass:smooth2}), the smooth evolution allows to identify the optimal arm. We remark that 
the order of $T$, $\tau$ and $\Delta_\tau$ matches the state-of-the-art results when applied to stationary bandits.}
Let us compare the previous results with the state-of-the-art ones in an environment characterized by Bernoulli rewards. The order for the regret is given by:
\begin{equation}
    R_T(\texttt{Beta-SWTS}/\texttt{$\gamma$-SWGTS})\leq O\left( \Delta'FT^{\beta} +\frac{T\ln(\tau)}{(\Delta'-2\sigma\tau)^2\tau}\right),
\end{equation}
matching the order of the regret obtained in Theorem~D.2 by~\citet{combes2014unimodal} for \texttt{SW-KL-UCB}.

\section{Experiments}\label{sec:exp}
We experimentally evaluate our algorithms w.r.t.~the state-of-the-art algorithms for NS-MABs. In particular, we considered the following baseline algorithms: \texttt{Rexp3}~\citep{besbes2014stochastic}, an NS-MAB algorithm based on variation budget, \texttt{SW-KL-UCB}~\citep{garivier2011kl}, one of the most effective stationary MAB algorithms, \texttt{Ser4}~\citep{allesiardo2017non}, which considers best arm switches during the process, and sliding-window algorithms that are generally able to deal with non-stationary bandit settings such as \texttt{SW-UCB}~\citep{garivier2011upper}, \texttt{SW-KL-UCB}~\citep{combes2014unimodal}.
We include an algorithm meant for stationary bandits, i.e., \texttt{TS} \citep{thompson1933ts}, to show the impact of the sliding window approach on the regret in dynamic scenarios.
The parameters for all the baseline algorithms have been set as recommended in the corresponding papers (see also Appendix~\ref{apx:experiments} for details). For all experiments, we consider $K = 10$ arms and set the learning horizon to $T = 5\cdot10^4$. The rewards for a chosen arm $i$ will be sampled from a Bernoulli distribution whose probability of success at time $t$ is given by $\mu_{i,t}$ that will evolve over rounds as specified in the following. Since we derived above that the order of cumulative regret for our algorithms is the same as that of \texttt{SW-UCB}, we set the window size $\tau$ for TS-like approaches to $\tau = 4 \sqrt{T \ln{T}}$, as also prescribed by \citet{garivier2008upper}.

Regarding our algorithms, we also provide a sensitivity analysis evaluating the cumulative regret for different choices of the window size $\tau$. We tested our algorithms assuming to misspecify the order of the sliding window w.r.t.~the learning horizon $T$, formally, we set $\alpha \in \{0.2, 0.4, 0.5, 0.6, 0.8\}$ and $\tau=T^{\alpha}$. For the sake of notation, we denote the theoretically-based choice for the parameter, i.e., $\tau = 4 \sqrt{T \ln{T}},$ as $\tau = T^{0.5}$ in the sensitivity analysis. We denote with $\alpha_{TS}$ the misspecification of the sliding window for \texttt{Betas-SWTS} and $\alpha_{GTS}$ the one for \texttt{$\gamma$-SWGTS}.

In the following, the results for the different algorithms $\mathfrak{A}$ are provided in terms of the empirical cumulated regret $\hat{R}_t(\mathfrak{A})$ averaged over $50$ independent runs. Standard deviations are provided as semi-transparent areas.

\subsection{Abruptly Changing Scenario}

In this scenario, we perform two experiments. First, we test the algorithms in a piecewise-constant, abruptly-changing setting. The evolution of the expected reward over time of the arms is provided in Figure~\ref{fig:abruptSetting}, and the formal definition of the expected reward evolution over phases is provided in Appendix~\ref{apx:experiments}. In the second experiment, we test the algorithms in a general abruptly-changing scenario, i.e., the expected rewards within each phase evolve arbitrarily between two breakpoints. The evolution of the expected rewards is represented in Figure \ref{fig:abrupsetting2}, and the formal definition of the expected reward evolution over time is provided in Appendix \ref{apx:experiments}. In both settings the optimal arm is ${10}$ during the $\mathcal{F}_1$ and $\mathcal{F}_3$ phases and arm $1$ during the $\mathcal{F}_2$ and $\mathcal{F}_4$ phases.

\paragraph{Results}
The results of the regret of the analyzed algorithms are provided in Figures~\ref{fig:abruptResult} and \ref{fig:abruptResult2}. Since similar conclusions can be drawn from both experiments, for the sake of presentation, we focus on the description of the former. The algorithms providing the worst performance overall are \texttt{Rexp3} and \texttt{Ser4}. We believe this can be explained by the way some hyperparameters are set based on theoretical considerations, which should be tuned depending on the specific scenario to provide better performance.
During the first phase $\mathcal{F}_1$, the best-performing algorithm is \texttt{TS}, since the setting is comparable to a stationary environment during the phase and it is the only algorithm considering the entire history to take decisions. As soon as we change phase, and consequently, the optimal arm changes, all the algorithms start accumulating regret at an increased rate. 
In particular, the \texttt{TS} algorithm cannot address this change, and its performance degrades as multiple changes occur.
Conversely, its sliding window counterpart \texttt{Beta-SWTS} provides the best performances starting from the initial part of phase $\mathcal{F}_2$ ($t \approx 12.000$), showing that forgetting the past is an effective strategy in such a scenario.
By the end of the learning horizon, most of the sliding-window-based approaches are able to outperform the \texttt{TS} algorithm.
The fact that $\texttt{$\gamma$-SWGTS}$ is not the best-performing algorithm in this setting is due to the fact that it is designed for generic subgaussian rewards, while the other ones are specifically crafted for Bernoulli rewards.
Therefore, in its design, it needs to introduce more exploration to deal with possibly more complex distribution than the Bernoulli.

\paragraph{Sensitivity Analysis}
Let us focus on the sensitivity analysis provided in Figure \ref{fig:abruptSensitivity} and \ref{fig:abruptSensitivity2}. In both environments, we see that for smaller window sizes, i.e., $\alpha = 0.2$, the algorithms become too explorative, leading to a larger regret at the end of the learning horizon. This means that we are too aggressive in discarding samples used for the arms' reward estimates, preventing the algorithms from converging to an optimum when the environment is not changing, i.e., we are not switching to the following phase. As the window size increases, the performance for both algorithms improves, achieving the minimum at the suggested window size (i.e., $\tau = 4\sqrt{T\log(T)}$) for \texttt{Beta-SWTS}, while \texttt{$\gamma$-SWGTS} reaches its best performance at $\alpha=0.8$, further highlighting the explorative nature of sampling from a Gaussian distribution in a Bernoulli setting.

\begin{figure}
\centering
\subfloat[]{
    \scalebox{1}{\includegraphics[]{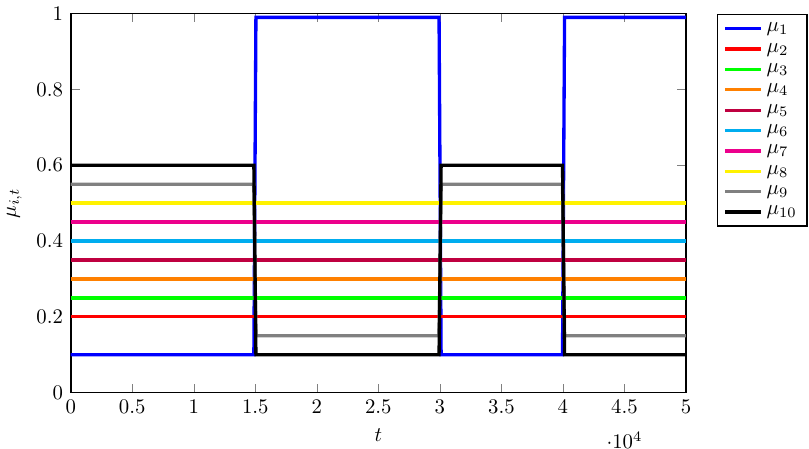}} \label{fig:abruptSetting}
} \\
\subfloat[]{
    \scalebox{1}{\includegraphics[]{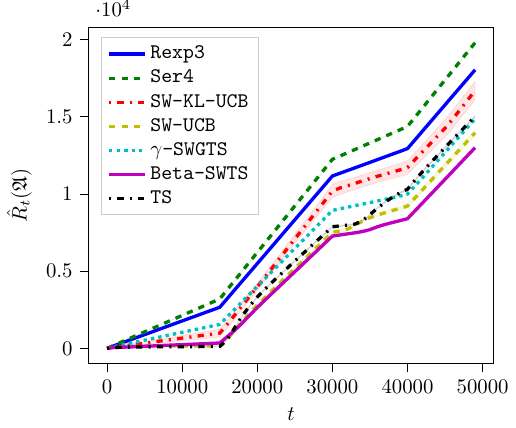}} \label{fig:abruptResult}
}
\subfloat[]{
    \scalebox{1}{\includegraphics[]{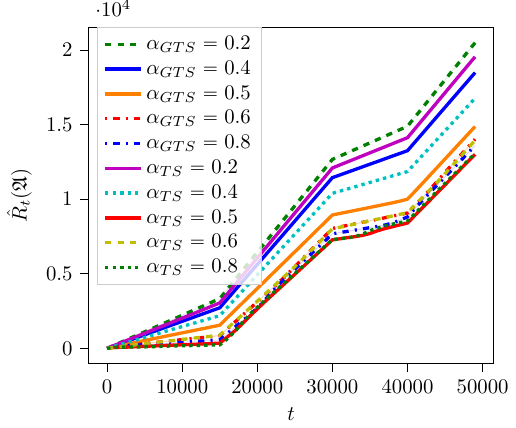}} \label{fig:abruptSensitivity}
}
    \caption{Abruptly Changing Scenario 1: (a) the abruptly changing environment, (b) cumulative regret comparison, (c) sensitivity analysis for the sliding window size.}
    \label{fig:abrupt}
\end{figure}

\begin{figure}
\centering
\subfloat[]{
        \scalebox{1}{\includegraphics[]{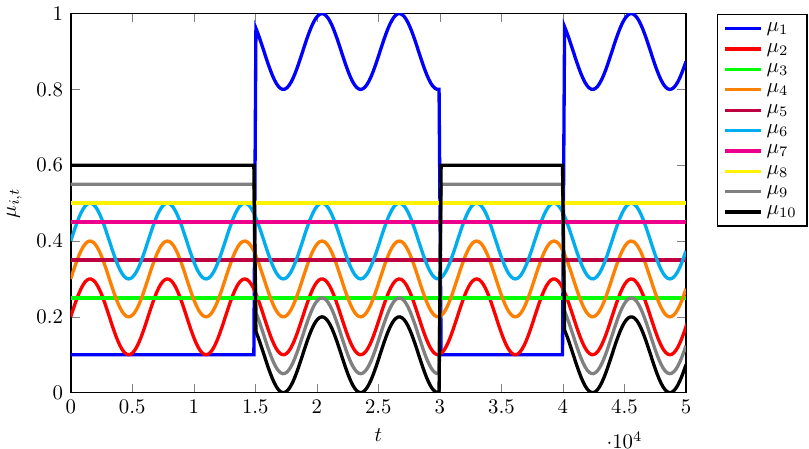}}
        \label{fig:abrupsetting2}
}\\
\subfloat[]{
        \scalebox{1}{\includegraphics[]{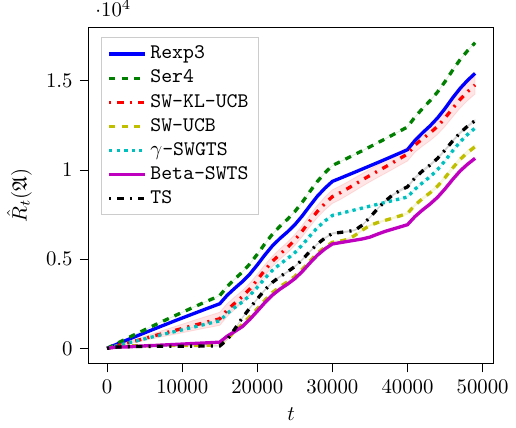}}
        \label{fig:abruptResult2}
}
\subfloat[]{
        \scalebox{1}{\includegraphics[]{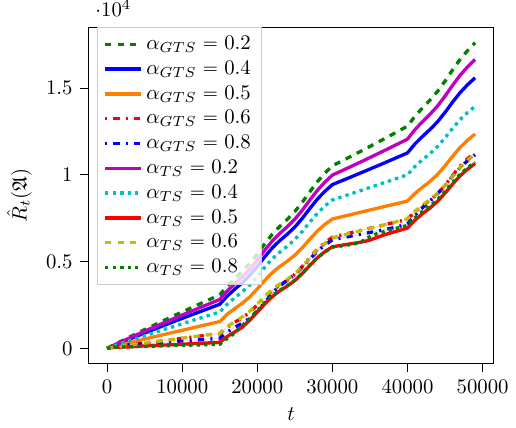}}
        \label{fig:abruptSensitivity2}
}
    \caption{Abruptly Changing Scenario 2: (a) the abruptly changing environment, (b) cumulative regret comparison, (c) sensitivity analysis for the sliding window size.}
    \label{fig:abrgen}
\end{figure}

\subsection{Smoothly Changing Scenario}

Similarly to what has been done by \citet{combes2014unimodal}, we test our algorithms on an instance of a smoothly changing environment, as depicted in Figure~\ref{fig:smoothenvironment}. In this setting, the smoothness parameter is set to $\sigma=0.0001$. We report the formal evolution of the expected reward and additional results on other smoothly changing environments with different values for the smoothness parameter $\sigma$ in Appendix~\ref{apx:experiments}.
Even in this environment, the optimal arm changes over time so that each arm is optimal for at least one round over the selected learning horizon.

\paragraph{Results}
The cumulative regret is provided in Figure~\ref{fig:smoothresultsexp}. Among the worst performing algorithms we have \texttt{Ser4}, \texttt{Rexp3}, and \texttt{SW-KL-UCB}. Even in this case, the issue is related to the initialization of the parameters that may play a crucial role in having low regret. In this setting \texttt{Beta-SWTS} outperforms all the other algorithms in $t \in [30.000, 50.000]$. Indeed, it is particularly effective in dealing with cases in which arms whose expected reward was among the lowest becomes optimal. For instance, in $t \in [10.000, 20.000]$, phase in which arm $a_{10}$ become optimal, the \texttt{Beta-SWTS} is providing the lowest increase rate among the analyzed algorithms.
Once more, the classical \texttt{TS} algorithm is outperformed by its sliding-window counterpart in $t \in [30.000, 50.000]$. Similarly to what happened in the generalized abruptly changing environments, the performance of \texttt{$\gamma$-SWGTS} displays moderate performance in this setting due to the more general formulation of the algorithm.

\paragraph{Sensitivity Analysis}
The sensitivity analysis is presented in Figure~\ref{fig:smoothresultssens}. The behavior is similar to what we presented in the abruptly-changing scenario. More specifically, for small sliding window sizes, the algorithms tend to explore more than is needed. Conversely, for larger values of the window size, the performance tends to collapse to almost the same regret curve. However, for $\alpha = 1$, i.e., using the classical \texttt{TS}, would provide a significantly large regret, which shows the necessity to introduce at least a limited amount of forgetting in such settings.

\begin{figure}[th!]
\centering
\subfloat[]{
\scalebox{1}{\includegraphics[]{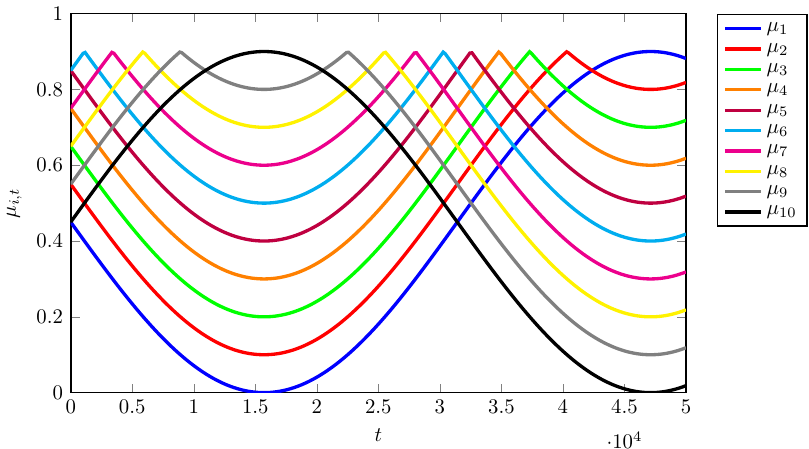}}
\label{fig:smoothenvironment}
}\\
\subfloat[]{
\scalebox{1}{\includegraphics[]{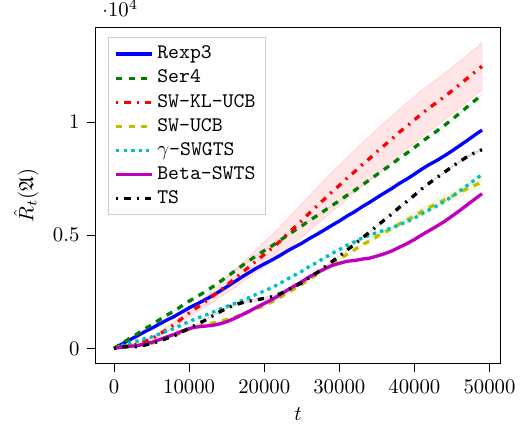}}
\label{fig:smoothresultsexp}
}
\subfloat[]{
\scalebox{1}{\includegraphics[]{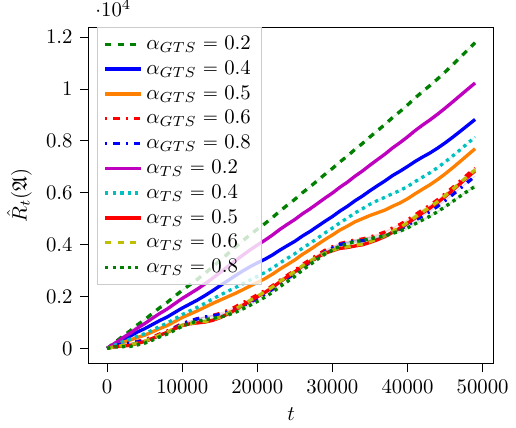}}
\label{fig:smoothresultssens}
}
\caption{Smoothly Changing Scenario: (a) the smoothly changing environment, (b) cumulative regret comparison, (c) sensitivity analysis for the sliding window size.}
\label{fig:smoothlyone}
\end{figure}

\section{Conclusions}

We have characterized the performance of TS-like algorithms designed for NS-MABs, namely \texttt{Beta-SWTS} and \texttt{$\gamma$-SWGTS}, in a general formulation for non-stationary setting, deriving general regret bounds to characterize the learning process in any arbitrary environment, for Bernoulli and subgaussian rewards, respectively. We have shown how such a general result applies to two of the most common non-stationary settings in the literature, namely the abruptly changing environment and the smoothly changing one, deriving upper bounds on the regret that are in line with the state of the art. Finally, we have performed numerical validations of the proposed algorithms against the baselines that represent the state-of-the-art solutions for learning in dynamic scenarios, showing how the sliding window approach applied to the TS algorithm is a viable solution to deal with Non-Stationary settings.

Future lines of research include developing specialized TS-like algorithms that take into account the \emph{specific} nature of the non-stationarity or extending the analysis to non-stationary cases in which the arms reward presents a structure among them, such as linear bandits.
\clearpage
\bibliographystyle{plainnat}
\bibliography{biblio}

\clearpage
\section*{Code Availability}

All the codes are publicly available at the following link: \url{https://github.com/albertometelli/stochastic-rising-bandits}.

\clearpage
\begin{appendices}
\section{Additional Lemmas}\label{apx:proofs}
We now present two Lemmas that will be useful troughout the analysis.
\begin{restatable}[]{definition}{probs}\label{def:prob}
Let $i,i' \in \dsb{K}$ be two arms, $t \in \dsb{T}$ be a round, $\tau \in \dsb{T}$ be the window, and $y_{i',t}\in (0,1)$ be a threshold, we define:
    \begin{equation}\label{eq:probprob}
    p_{i,t,\tau}^{i'} \coloneqq \mathbb{P}\left(\theta_{i,t,\tau}>y_{i',t}\right|\mathcal{F}_{t-1}),
    \end{equation}
    where $\mathcal{F}_t$ is the filtration induced by the sequence of arms played and observed rewards up to round $t$.
\end{restatable}
\begin{restatable}[]{definition}{probs}\label{def:rrounds}
For each $i\in \dsb{K}$, we define the set of rounds $t\in \mathcal{F}_\tau^\complement$ and $i\neq i^*(t)$ as $\mathcal{F}_{i,\tau}^\complement$. Formally:
\begin{equation}
    \mathcal{F}_{i,\tau}^\complement\coloneqq \mathcal{F}_{\tau}^\complement \cap\{t\in \dsb{T}:i\neq i^*(t)\}.
\end{equation}
\end{restatable}
We propose a slight modification of Lemma 5.1 from \cite{fiandri2025thompsonsamplinglikealgorithmsstochastic} and Lemma C.1 from \cite{fiandri2025thompsonsamplinglikealgorithmsstochastic}, to obtain results that are more suitable to describe the regret in restless setting.
\begin{restatable}[Expected Number of Pulls Bound for \texttt{Beta-SWTS}]{lemma}{rd} \label{lem:RD}
Let $T \in \mathbb{N}$ be the learning horizon, $\tau \in \dsb{T}$ the window size, for the \texttt{Beta-SWTS} algorithm it holds for every free parameter $\omega \in \dsb{0,T}$ that:
    \begin{align*}
        \mathbb{E} & [N_{i,T}] \leq |\mathcal{F}_\tau|+\frac{T}{\tau} +\mathbb{E} \left[ \sum_{t\in \mathcal{F}_{i,\tau}^\complement} \mathds{1} \left\{ p_{i,t,\tau}^i > \frac{1}{\tau},  \ N_{i,t,\tau} \ge \omega \right\} \right] +\frac{\omega T}{\tau} +\mathbb{E} \left[ \sum_{t\in \mathcal{F}_{i,\tau}^\complement} \left( \frac{1}{p_{i^*(t),t,\tau}^i}-1 \right) \mathds{1} \left\{ I_t=i^*(t) \right\} \right].
    \end{align*}
\end{restatable}
\begin{proof}
    The proof will follow the same steps of the proof in \cite{fiandri2025thompsonsamplinglikealgorithmsstochastic} with some changes to adapt to the restless setting.
    We define the event $E_i(t)\coloneqq \{\theta_{i,t,\tau}\leq y_{i,t}\}$. Thus, assigning immediate regret equal to one for every round in $\mathcal{F}_\tau$ the following holds:
\begin{equation}
    \mathbb{E}[N_{i,T}]=\sum_{t=1}^{T}\Prob(I_t=i, i\neq i^*(t))\leq |\mathcal{F}_\tau|+\underbrace{\sum_{t\in\mathcal{F}_{i,\tau}^\complement}\Prob(I_t=i, E_i^\complement(t))}_{\text{(A)}}+\underbrace{\sum_{t\in\mathcal{F}_{i,\tau}^\complement} \Prob(I_t=i, E_i(t))}_{\text{(B)}}.
\end{equation}
Let us first face term (A): 
\begin{align}
    \text{(A)} &\leq\sum_{t\in\mathcal{F}_{i,\tau}^\complement} \Prob(I_t=i, E_i^\complement(t), N_{i,t,\tau}\leq \omega)+\sum_{t\in\mathcal{F}_{i,\tau}^\complement}\Prob(I_t=i, E_i^\complement(t), N_{i,t,\tau}\ge \omega)\\
    &\leq \sum_{t\in\mathcal{F}_{i,\tau}^\complement}\Prob(I_t=i, N_{i,t,\tau}\leq \omega)+\sum_{t\in\mathcal{F}_{i,\tau}^\complement}\Prob(I_t=i, E_i^\complement(t), N_{i,t,\tau}\ge \omega)\\
    &= \mathbb{E}\left[\sum_{t\in\mathcal{F}_{i,\tau}^\complement}\mathds{1}\left\{I_t=i,N_{i,t,\tau}\leq\omega\right\}\right]+\sum_{t\in\mathcal{F}_{i,\tau}^\complement}\Prob(I_t=i, E_i^\complement(t), N_{i,t,\tau}\ge \omega) \\
&\leq\mathbb{E}\left[\underbrace{\sum_{t=1}^{T}\mathds{1}\left\{I_t=i, N_{i,t,\tau}\leq \omega\right\}}_{\text{(C)}}\right]+\sum_{t\in\mathcal{F}_{i,\tau}^\complement}\Prob(I_t=i, E_i^\complement(t), N_{i,t,\tau}\ge \omega).
\end{align}
Observe that (C) can be bounded by Lemma \ref{lemma:window}. Thus, the above inequality can be rewritten as:
\begin{align}
    \text{(A)} &\leq \frac{\omega T}{\tau}+\underbrace{\sum_{t\in\mathcal{F}_{i,\tau}^\complement}\Prob(I_t=i, E_i^\complement(t), N_{i,t,\tau}\ge \omega)}_{\text{(D)}}.
\end{align}
We now focus on the term (D). Defining $\mathcal{T}:=\{t \in \mathcal{F}_{i,\tau}^\complement:1-\Prob(\theta_{i,t,\tau}\leq y_{i,t}\mid \mathcal{F}_{t-1})> \frac{1}{\tau}, N_{i,t,\tau}\ge \omega \}$ and $\mathcal{T}':=\{t \in\mathcal{F}_{i,\tau}^\complement:1-\Prob(\theta_{i,t,\tau}\leq y_{i,t}\mid \mathcal{F}_{t-1})\le \frac{1}{\tau}, N_{i,t,\tau}\ge \omega \}$ we obtain:
\begin{align}
    \sum_{t\in\mathcal{F}_{i,\tau}^\complement} & \Prob(I_t=i, E_i^\complement(t), N_{i,t,\tau}\ge \omega) = \mathbb{E}\left[\sum_{t\in\mathcal{F}_{i,\tau}^\complement}\mathds{1}\{I_t=i, E_i^\complement(t), N_{i,t,\tau}\ge \omega\}\right]\\
    &=\mathbb{E}\left[\sum_{t \in \mathcal{T}}\mathds{1}\{I_t=i,E_i(t)^\complement\}\right]+\mathbb{E}\left[\sum_{t \in \mathcal{T}'}\mathds{1}\{I_t=i,E_i(t)^\complement\}\right]\\
    &\leq \mathbb{E}\left[\sum_{t \in \mathcal{T}}\mathds{1}\{I_t=i\}\right]+\mathbb{E}\left[\sum_{t \in \mathcal{T}'}\mathds{1}\{E_i(t)^\complement\}\right]\\
    &\leq \mathbb{E}\left[\sum_{t\in\mathcal{F}_{i,\tau}^\complement}\mathds{1}\left\{1-\Prob(\theta_{i,t,\tau}\leq y_{i,t}\mid \mathcal{F}_{t-1})> \frac{1}{\tau}, N_{i,t,\tau}\ge \omega, I_t=i\right\}\right]+\sum_{t=1}^T\frac{1}{\tau}.
\end{align}
Now we focus on term (B). We have:
\begin{align}
  \sum_{t\in\mathcal{F}_{i,\tau}^\complement} \Prob(I_t=i, E_i(t))=\sum_{t\in\mathcal{F}_{i,\tau}^\complement} \mathbb{E}\left[\underbrace{\Prob(I_t=i, E_i(t)\mid \mathcal{F}_{t-1})}_{\text{(E)}}\right].
\end{align}
In order to bound (B) we need to bound (E).
Let $i_t^{\prime}=\operatorname{argmax}_{i \neq i^*(t)} \theta_{i,t,\tau}$. Then, we have:
$$
\begin{aligned}
\mathbb{P}\left(I_t=i^*(t), E_i(t) \mid \mathcal{F}_{t-1}\right) & \geq \mathbb{P}\left(i_t^{\prime}=i, E_i(t), \theta_{i^*(t),t,\tau} > y_{i,t} \mid \mathcal{F}_{t-1}\right) \\
& =\mathbb{P}\left(\theta_{i^*(t),t,\tau} > y_{i,t} \mid \mathcal{F}_{t-1}\right) \mathbb{P}\left(i_t^{\prime}=i, E_i(t) \mid \mathcal{F}_{t-1}\right) \\
& \geq \frac{p_{i^*(t),t,\tau}^i}{1-p_{i^*(t),t,\tau}^i} \mathbb{P}\left(I_t=i, E_i(t) \mid \mathcal{F}_{t-1}\right),
\end{aligned}
$$
where in the first equality we used the fact that $\theta_{i^*(t),t,\tau}$ is conditionally independent of $i_t^{\prime}$ and $E_i(t)$ given $\mathcal{F}_{t-1}$. In the second inequality, we used the fact that:
$$
\mathbb{P}\left(I_t=i, E_i(t) \mid \mathcal{F}_{t-1}\right) \leq\left(1-\mathbb{P}\left(\theta_{i^*(t),t,\tau}>y_{i,t} \mid \mathcal{F}_{t-1}\right)\right) \mathbb{P}\left(i_t^{\prime}=i, E_i(t) \mid \mathcal{F}_{t-1}\right),
$$
which is true since $\left\{I_t=i\right\} \cap E_i(t)  \subseteq\left\{i_t^{\prime}=i\right\} \cap E_i(t) \cap\left\{\theta_{i^*(t),t,\tau}\leq y_{i,t}\right\}$, and the two intersected events are conditionally independent given $\mathcal{F}_{t-1}$. Therefore, we have:
$$
\begin{aligned}
\mathbb{P}\left(I_t=i, E_i(t) \mid \mathcal{F}_{t-1}\right) & \leq\left(\frac{1}{p_{i^*(t),t,\tau}^i}-1\right) \mathbb{P}\left(I_t=i^*(t), E_i(t) \mid \mathcal{F}_{t-1}\right) \\
& \leq\left(\frac{1}{p_{i^*(t),t,\tau}^i}-1\right) \mathbb{P}\left(I_t=i^*(t) \mid \mathcal{F}_{t-1}\right),
\end{aligned}
$$
substituting, we obtain:
\begin{align}
    \sum_{t\in\mathcal{F}_{i,\tau}^\complement}\mathbb{E}[\Prob(I_t=i,E_i(t)\mid\mathcal{F}_{t-1})]&\leq \mathbb{E}\left[\sum_{t\in\mathcal{F}_{i,\tau}^\complement}\left(\frac{1}{p_{i^*(t),t,\tau}^i}-1\right) \mathbb{P}\left(I_t=i^*(t) \mid \mathcal{F}_{t-1}\right)\right]\\
    &=\mathbb{E}\left[\mathbb{E}\left[\sum_{t\in\mathcal{F}_{i,\tau}^\complement}\left(\frac{1}{p_{i^*(t),t,\tau}^i}-1\right)\mathds{1}\left\{I_t=i^*(t)\right\}\mid\mathcal{F}_{t-1}\right]\right]\\
    &=\mathbb{E}\left[\sum_{t\in\mathcal{F}_{i,\tau}^\complement} \left(\frac{1}{p_{i^*(t),t,\tau}^i}-1\right)\mathds{1}\left\{I_t=i^*(t)\right\}\right].
\end{align}
The statement follows by summing all the terms.
\end{proof}
\begin{restatable}[Expected Number of Pulls Bound for  \texttt{$\gamma$-SWGTS}]{lemma}{rdg} \label{lem:RDg}
   Let $T \in \mathbb{N}$ be the learning horizon, $\tau \in \dsb{T}$ be the window size, for the  \texttt{$\gamma$-ET-SWGTS} algorithm the following holds for every $i \neq i^*(t)$ and free parameters $\omega \in \dsb{T}$ and $\epsilon > 0$:
     \begin{align*}
        \mathbb{E}[N_{i,T}] & \leq |\mathcal{F}_\tau|+\frac{T}{\tau\epsilon_i}+\frac{T}{\tau}+\frac{\omega T}{\tau}+{\mathbb{E}\left[\sum_{t\in \mathcal{F}_{i,\tau}^\complement}\mathds{1}\left \{ p_{i,t,\tau}^i>\frac{1}{\tau\epsilon_i}, \ N_{i,t,\tau}\ge \omega\right\}\right]} +{\mathbb{E}\left[\sum_{t\in \mathcal{F}_{i,\tau}^\complement}\left(\frac{1}{p_{i^*(t),t,\tau}^i}-1\right)\mathds{1}\left\{I_t=i^*(t) \right\}\right]}.
    \end{align*}
\end{restatable}
\begin{proof}
    We define the event $E_i(t)\coloneqq \{\theta_{i,t,\tau}\leq y_{i,t}\}$. Thus, the following holds, assigning "error" equal to one for every round in $\mathcal{F}_\tau$:
\begin{equation}
    \mathbb{E}[N_{i,T}]=\sum_{t=1}^{T}\Prob(I_t=i,i\neq i^*(t))\leq |\mathcal{F}_\tau|+\underbrace{\frac{T}{\tau}}_{\text{(X)}}+\underbrace{\sum_{t\in\mathcal{F}_{i,\tau}^\complement} \Prob(I_t=i, E_i^\complement(t))}_{\text{(A)}}+\underbrace{\sum_{t=\mathcal{F}_{i,\tau}^\complement} \Prob(I_t=i, E_i(t))}_{\text{(B)}},
\end{equation}
where (X) is the term arising given by the forced play whenever $N_{i,t,\tau}=0$. 
Let us first face term (A): 
\begin{align}
    \text{(A)} &\leq\sum_{t\in\mathcal{F}_{i,\tau}^\complement} \Prob(I_t=i, E_i^\complement(t), N_{i,t,\tau}\leq \omega)+\sum_{t\in \mathcal{F}_{i,\tau}^\complement}\Prob(I_t=i, E_i^\complement(t), N_{i,t,\tau}\ge \omega)\\
    &\leq \sum_{t\in \mathcal{F}_{i,\tau}^\complement}\Prob(I_t=i, N_{i,t,\tau}\leq \omega)+\sum_{t\in \mathcal{F}_{i,\tau}^\complement}\Prob(I_t=i, E_i^\complement(t), N_{i,t,\tau}\ge \omega)\\
    &\leq \mathbb{E}\left[\sum_{t\in \mathcal{F}_{i,\tau}^\complement}\mathds{1}\left\{I_t=i, N_{i,t,\tau}\leq \omega\right\}\right]+\sum_{t\in \mathcal{F}_{i,\tau}^\complement}\Prob(I_t=i, E_i^\complement(t), N_{i,t,\tau}\ge \omega)\\
&\leq\mathbb{E}\left[\underbrace{\sum_{t=1}^{T}\mathds{1}\left\{I_t=i, N_{i,t,\tau}\leq \omega\right\}}_{\text{(C)}}\right]+\sum_{t\in \mathcal{F}_{i,\tau}^\complement}\Prob(I_t=i, E_i^\complement(t), N_{i,t,\tau}\ge \omega).
\end{align}
Observe that (C) can be bounded by Lemma \ref{lemma:window}. Thus, the above inequality can be rewritten as:
\begin{align}
    \text{(A)} &\leq \frac{\omega T}{\tau}+\underbrace{\sum_{t=1}^{T}\Prob(I_t=i, E_i^\complement(t), N_{i,t,\tau}\ge \omega)}_{\text{(D)}}.
\end{align}
We now focus on the term (D). Defining $\mathcal{T}:=\{t \in \mathcal{F}_{i,\tau}^\complement:1-\Prob(\theta_{i,t,\tau}\leq y_{i,t}\mid \mathcal{F}_{t-1})> \frac{1}{\tau\epsilon_i}, N_{i,t,\tau}\ge \omega \}$ and $\mathcal{T}':=\{t \in\mathcal{F}_{i,\tau}^\complement:1-\Prob(\theta_{i,t,\tau}\leq y_{i,t}\mid \mathcal{F}_{t-1})\le \frac{1}{\tau\epsilon_i}, N_{i,t,\tau}\ge \omega \}$ we obtain:
\begin{align}
    \sum_{t\in\mathcal{F}_{i,\tau}^\complement}\Prob(I_t=i, E_i^\complement(t), N_{i,t,\tau}\ge \omega) &= \mathbb{E}\left[\sum_{t\in \mathcal{F}_{i,\tau}^\complement}\mathds{1}\{I_t=i, E_i^\complement(t), N_{i,t,\tau}\ge \omega\}\right]\\
    &=\mathbb{E}\left[\sum_{t \in \mathcal{T}}\mathds{1}\{I_t=i,E_i(t)^\complement\}\right]+\mathbb{E}\left[\sum_{t \in \mathcal{T}'}\mathds{1}\{I_t=i,E_i(t)^\complement\}\right]\\
    &\leq \mathbb{E}\left[\sum_{t \in \mathcal{T}}\mathds{1}\{I_t=i\}\right]+\mathbb{E}\left[\sum_{t \in \mathcal{T}'}\mathds{1}\{E_i(t)^\complement\}\right]\\
    &\leq \mathbb{E}\left[\sum_{t\in\mathcal{F}_{i,\tau}^\complement}\mathds{1}\left\{1-\Prob(\theta_{i,t,\tau}\leq y_{i,t}\mid \mathcal{F}_{t-1})> \frac{1}{\tau\epsilon_i}, N_{i,t,\tau}\ge \omega, I_t=i \right\}\right]+\sum_{t=1}^T\frac{1}{\tau\epsilon_i}.
\end{align}
Term (B) is bounded exactly as in the proof of Lemma~\ref{lem:RD}. The statement follows by summing all the terms.
\end{proof}

\section{Proofs}\label{apx:proofs2}
\swbetagen*
\begin{proof}
First of all, let us recall Lemma \ref{lem:RD}:
\begin{align*}
        \mathbb{E} [N_{i,T}] \leq |\mathcal{F}_\tau|+\frac{T}{\tau} +\underbrace{\mathbb{E}\left[\sum_{t\in\mathcal{F}_{i,\tau}^\complement} \mathds{1} \left\{ p_{i,t,\tau}^i > \frac{1}{\tau},  \ N_{i,t,\tau} \ge \omega \right\} \right]}_{(\mathcal{S}.1)} +\frac{\omega T}{\tau} +\underbrace{\mathbb{E} \left[ \sum_{t\in\mathcal{F}_{i,\tau}^\complement} \left( \frac{1}{p_{i^*(t),t,\tau}^i}-1 \right) \mathds{1} \left\{ I_t=i^*(t) \right\} \right]}_{(\mathcal{S}.2)}.
    \end{align*}
     Let us define the two threshold quantities $x_{i,t}$ and $y_{i,t}$ for $t \in \mathcal{F}_{i,\tau}^{\complement}$ ($t$ being the time the policy-maker has to choose the arm) as:
     \begin{equation}
         \max_{t' \in \dsb{t-\tau,t-1}}\{\mu_{i,t'}\} <x_{i,t}<y_{i,t}<\min_{t' \in \dsb{t-\tau,t-1}}\{\mu_{i^*(t),t'}\}
     \end{equation}
     with $\Delta_{i,t,\tau}=\min_{t' \in \dsb{t-\tau,t-1}}\{\mu_{i^*(t),t'}\}-\max_{t' \in \dsb{t-1,t-\tau}}\{\mu_{i(t),t'}\}$, we will always consider in the following analysis the choices: 
    $$
    x_{i,t}=\max_{t' \in \dsb{t-\tau,t-1}}\{\mu_{i(t),t'}\}+\frac{\Delta_{i,t,\tau}}{3},
    $$
    $$
     y_{i,t}=\min_{t' \in \dsb{t-\tau,t-1}}\{\mu_{i^*(t),t'}\}-\frac{\Delta_{i,t,\tau}}{3}.
    $$
    Notice then that the following quantities will have their minima for those $t \in \mathcal{F}_{\tau}^{\complement}$ such $\Delta_{i,t,\tau}=\Delta_{\tau}$:
    \begin{equation}
        \begin{rcases}
          &y_{i,t}-x_{i,t}\\
          &x_{i,t}-\max_{t' \in \dsb{t-\tau,t-1}}\{\mu_{i(t),t'}\}\\
          &\min_{t' \in \dsb{t-\tau,t-1}}\{\mu_{i^*(t),t'}\}-y_{i,t}         
        \end{rcases}
        \text{$=\frac{\Delta_{i,t,\tau}}{3}\ge \frac{\Delta_{\tau}}{3}$,}
    \end{equation}
and independently from the time $t \in \dsb{T}$ in which happens, they will always have the same value. We refer to the minimum values the quantities above can get in 
$t\in \mathcal{F}_{i,\tau}^{\complement}$ as:
\begin{equation}
    \begin{rcases}
        &y_i-x_i\\
        &x_i-\mu_{i,\mathcal{F}_{\tau}^{\complement}}\\
        &\mu_{i^*,\mathcal{F}_{\tau}^{\complement}}-y_i
    \end{rcases}
    \text{$=\frac{\Delta_{\tau}}{3}$.}
\end{equation}
We choose $\omega=\frac{\ln(\tau)}{2(x_i-y_i)^2}$ and define $\hat{\mu}_{i,t,\tau}=\frac{S_{i,t,\tau}}{N_{i,t,\tau}}$. We will consider $\tau\ge e$.
We first tackle Term ($\mathcal{S}.1$).
$\\$
\paragraph{Term ($\mathcal{S}.1$)}

We have:
\begin{align}
    (\mathcal{S}.1) &= \mathbb{E}\left[\sum_{t\in\mathcal{F}_{i,\tau}^\complement} \mathds{1} \left\{ p_{i,t,\tau}^i > \frac{1}{\tau},  \ N_{i,t,\tau} \ge \omega \right\} \right] \\
    & \leq \mathbb{E} \left[ \sum_{t\in \mathcal{F}^{\complement}_{i,\tau}}  \mathds{1} \left\{ p_{i,t,\tau}^i > \frac{1}{\tau},  \ N_{i,t,\tau} \ge \omega, \ \hat{\mu}_{i,t,\tau}\leq x_{i,t} \right\}\right]+\mathbb{E}\left[ \sum_{t\in \mathcal{F}^{\complement}_{i,\tau}}  \mathds{1} \left\{ p_{i,t,\tau}^i > \frac{1}{\tau},  \ N_{i,t,\tau} \ge \omega, \ \hat{\mu}_{i,t,\tau}\ge x_{i,t} \right\}\right]\\
    &\leq  \underbrace{\mathbb{E} \left[ \sum_{t\in \mathcal{F}^{\complement}_{i,\tau}}  \mathds{1} \left\{\overbrace{ p_{i,t,\tau}^i > \frac{1}{\tau},  \ N_{i,t,\tau} \ge \omega, \ \hat{\mu}_{i,t,\tau}\leq x_{i,t}}^{(*)} \right\}\right]}_{(\mathcal{S}.1.1)}+ \underbrace{{\sum_{t\in \mathcal{F}^{\complement}_{i,\tau}}  \Prob\left(N_{i,t,\tau} \ge \omega, \ \hat{\mu}_{i,t,\tau}\ge x_{i,t} \right)}}_{(\mathcal{S}.1.2)}.
\end{align}
 First, we face term $(\mathcal{S}.1.2)$, for each summand in the sum holds the following:
\begin{align}
   \Prob\left(N_{i,t,\tau} \ge \omega, \ \hat{\mu}_{i,t,\tau}\ge x_{i,t} \right) &\leq \Prob(\hat{\mu}_{i,t,\tau}\ge x_{i,t} \ | \ N_{i,t,\tau} \ge \omega)\\
   &\leq \Prob(\hat{\mu}_{i,t,\tau}-\mathbb{E}[\hat{\mu}_{i,t,\tau}]\ge x_{i,t}-\mathbb{E}[\hat{\mu}_{i,t,\tau}] \ | \ N_{i,t,\tau} \ge \omega)\\
   &\leq \Prob(\hat{\mu}_{i,t,\tau}-\mathbb{E}[\hat{\mu}_{i,t,\tau}]\ge x_i-\mu_{i,\mathcal{F}_{\tau}^{\complement}} \ | \ N_{i,t,\tau} \ge \omega)\label{eq:ch1}\\
   &\leq \exp(-2N_{i,t,\tau}(x_i-\mu_{i,\mathcal{F}_{i,\tau}^\complement})^2)|_{N_{i,t,\tau}\ge\omega}\label{eq:ch2}\\
   &\leq \frac{1}{\tau},
\end{align}
where the inequality from Equation \eqref{eq:ch1} to Equation \eqref{eq:ch2} follow from the Chernoff-Hoeffding inequality.
Summing over all the round $t$, we obtain $(\mathcal{S}.1.2)\leq\frac{T}{\tau}$
{We now focus on term $(\mathcal{S}.1.1)$. We want to assess if it is possible for condition $(*)$ to happen, in order to do so evaluate the following:}
\begin{align}
    \Prob& (\theta_{i,t,\tau}>y_{i,t}|N_{i,t,\tau}\ge\omega, \ \hat{\mu}_{i,t,\tau} \leq x_{i,t}, \  \mathbb{F}_{t-1}) \\
     & = \Prob \left( \text{Beta} \left( \hat{\mu}_{i,t,\tau} N_{i,t,\tau} + 1, (1 - \hat{\mu}_{i,t,\tau}) N_{i,t,\tau} + 1 \right) > y_{i,t} | N_{i,t,\tau}\ge\omega, \ \hat{\mu}_{i,t,\tau} \leq x_{i,t} \right) \\
    & \leq \Prob \left( \text{Beta} \left( x_{i,t} N_{i,t,\tau} + 1, (1 - x_{i,t}) N_{i,t,\tau} + 1 \right) > y_{i,t}|N_{i,t,\tau}\ge\omega \right) \label{line:lline1}\\
    &\leq F^{B}_{N_{i,t,\tau}+1,y_{i,t}}\big(x_{i,t}N_{i,t,\tau}|N_{i,t,\tau}\ge\omega\big)\label{eq:line2}\\
    & \leq F^{B}_{N_{i,t,\tau},y_{i,t}}\big(x_{i,t}N_{i,t,\tau}|N_{i,t,\tau}\ge\omega\big) \label{eq:line3}\\
    & \leq\exp \left( - N_{i,t,\tau} d(x_{i,t}, y_{i,t})\right) |_{N_{i,t,\tau}\ge\omega}\\
    &\leq \exp \left( - 2\omega (y_i-x_i)^2\right),
\end{align}
where for the last inequality, we exploited the Pinsker inequality. Equation~\eqref{line:lline1} was derived by exploiting the fact that on the event $x_{i,t}\ge\hat{\mu}_{i,t,\tau}$ a sample from $\text{Beta} \left( x_{i,t} N_{i,t,\tau} + 1, (1 - x_{i,t}) N_{i,t,\tau} + 1 \right) $ is likely to be as large as a sample from $\text{Beta}( \hat{\mu}_{i,t,\tau} N_{i,t,\tau} + 1, (1 - \hat{\mu}_{i,t,\tau})N_{i,t,\tau} + 1 )$, reported formally in Lemma~\ref{lem:nbord}. Equation \eqref{eq:line2} follows from Fact \ref{lem:betabin}, while Equation \ref{eq:line3} from Lemma \ref{lem:bborder}
Therefore, for $\omega = \frac{\log{\tau}}{2(y_i-x_i)^2}$ we have:
\begin{equation}\label{eq:rif}
   \Prob(\theta_{i,t,\tau}>y_{i,t}|\ N_{i,t,\tau}\ge\omega,\ \hat{\mu}_{i,t,\tau} \leq x_{i,t}, \mathbb{F}_{t-1}) \leq\frac{1}{\tau}.
\end{equation}
Then, it follows that condition $(*)$ is never met, and each summand in $(\mathcal{S}.1.1)$ is equal to zero, so $(\mathcal{S}.1.1)=0$.
\paragraph{Term ($\mathcal{S}.2$)}
We can rewrite the term ($\mathcal{S}.2$) as follows:
\begin{align}
    \mathbb{E} \left[ \sum_{t\in \mathcal{F}_{i,\tau}^\complement} \left( \frac{1}{p_{i^*(t),t,\tau}^i}-1 \right) \mathds{1} \left\{ I_t=i^*(t) \right\} \right]  & = \sum_{t\in \mathcal{F}_{i,\tau}^\complement}\mathbb{E} \left[ \left( \frac{1}{p_{i^*(t),t,\tau}^i}-1 \right) \mathds{1} \left\{ I_t=i^*(t) \right\} \right]\\
    &=\underbrace{\sum_{t\in \mathcal{F}^{\complement}_{i,\tau}}\mathbb{E}\left[\left( \frac{1}{p_{i^*(t),t,\tau}^i}-1 \right) \mathds{1} \left\{ \overbrace{I_t=i^*(t), \ N_{i^*(t),t,\tau}\le8\frac{\log(\tau)}{(\mu_{i^*,\mathcal{F}_{i,\tau}^\complement}-y_{i})^2}}^{\mathcal{C}1} \right\} \right]}_{(\mathcal{S}.2.1)}+\nonumber\\ &+\underbrace{\sum_{t\in \mathcal{F}^{\complement}_{i,\tau}}\mathbb{E}\left[ \left( \frac{1}{p_{i^*(t),t,\tau}^i}-1 \right) \mathds{1} \left\{\overbrace{ I_t=i^*(t), \ N_{i^*(t),t,\tau}> 8\frac{\log(\tau)}{(\mu_{i^*,\mathcal{F}_{i,\tau}^\complement}-y_{i})^2}}^{\mathcal{C}2} \right\} \right]}_{(\mathcal{S}.2.2)}.
\end{align}
Exploiting the fact that $\mathbb{E}[XY]=\mathbb{E}[X\mathbb{E}[Y\mid X]]$ we can rewrite both $(\mathcal{S}.2.1)$ and $(\mathcal{S}.2.2)$ as:
\begin{align}
    (\mathcal{S}.2.1)&=\sum_{t\in \mathcal{F}^{\complement}_{i,\tau}} \mathbb{E}\left[\mathds{1}\{\mathcal{C}1\}\mathbb{E}\left[\left(\frac{1}{p_{i^*(t),t,\tau}^i}-1\right)\bigg| \mathcal{C}1\right]\right]=\mathbb{E}\left[\sum_{t\in \mathcal{F}^{\complement}_{i,\tau}} \mathds{1}\{\mathcal{C}1\}\mathbb{E}\left[\left(\frac{1}{p_{i^*(t),t,\tau}^i}-1\right) \bigg| \mathcal{C}1\right]\right],\\
    (\mathcal{S}.2.2)&=\sum_{t\in \mathcal{F}^{\complement}_{i,\tau}} \mathbb{E}\left[\mathds{1}\{\mathcal{C}2\}\mathbb{E}\left[\left(\frac{1}{p_{i^*(t),t,\tau}^i}-1\right) \bigg| \mathcal{C}2 \right]\right]=\mathbb{E}\left[\sum_{t\in \mathcal{F}^{\complement}_{i,\tau}}\mathds{1}\{\mathcal{C}2\}\mathbb{E}\left[\left(\frac{1}{p_{i^*(t),t,\tau}^i}-1\right) \bigg| \mathcal{C}2\right]\right].
\end{align}
Let us first tackle term ($\mathcal{S}.2.1$):
\begin{align}\label{eq:A}
    (\mathcal{S}.2.1)=\mathbb{E}\left[\sum_{t\in \mathcal{F}^{\complement}_{i,\tau}} \mathds{1}\{\mathcal{C}1\}\mathbb{E}\left[\left(\frac{1}{p_{i^*(t),t,\tau}^i}-1 \right)\bigg| \mathcal{C}1\right]\right].
\end{align}
Taking inspiration from peeling-like arguments, let us decompose the event $\mathcal{C}1$ in $\lceil\log(\tau)\rceil$ sub-events $\mathcal{C}1_j$ for $j\ge 1 $ defined as follow:
\begin{align}
    \{\mathcal{C}1_j\}=\left\{\underbrace{e^{j-1}}_{\coloneqq N_{j-1}}< N_{i^*(t),t,\tau}\le \underbrace{e^j}_{\coloneqq N_j}, \ I_t=i^*(t)\right\},
\end{align}
with the convention:
\begin{align}
  \{\mathcal{C}1_1\}=\left\{\underbrace{0}_{\coloneqq N_0}\le N_{i^*(t),t,\tau}\le \underbrace{e}_{\coloneqq N_1}, \ I_t=i^*(t) \right\}.
\end{align}
notice that $\lceil\log(\tau) \rceil$ of such sub-events are enough as by definition $N_{i,t,\tau}\le\tau$ holds. This yields to:
\begin{align}
    \mathds{1}\{\mathcal{C}1\}\le \sum_{j=1}^{\lceil\log(\tau)\rceil}\mathds{1}\{\mathcal{C}1_j\}.
\end{align}
Let $\Delta_i'\coloneqq\mu_{i^*,\mathcal{F}_\tau^\complement}-y_i$, we can rewrite term ($\mathcal{S}.2.1$) as:
\begin{align}
    (\mathcal{S}.2.1) &\le\mathbb{E}\left[\sum_{j=1}^{\lceil\log(\tau)\rceil}\sum_{t\in \mathcal{F}_{i,\tau}^\complement}\mathds{1}\{\mathcal{C}1_j\}\mathbb{E}\left[\left(\frac{1}{p_{i^*(t),t,\tau}^i}-1\right)\bigg|\mathcal{C}1\right]\right]\\
    &= \mathbb{E}\left[\underbrace{\sum_{j=1}^{\lceil\log(\frac{8}{\Delta_i'})\rceil}\hspace{-0.3 cm}\sum_{t\in \mathcal{F}_{i,\tau}^\complement}\mathds{1}\{\mathcal{C}1_j\}\mathbb{E}\left[\left(\frac{1}{p_{i^*(t),t,\tau}^i}-1\right)\bigg|\mathcal{C}1\right]}_{(A)}\right]+\mathbb{E}\left[\underbrace{\sum_{j=\lceil\log(\frac{8}{\Delta_i'})\rceil+1}^{\lceil\log(\tau)\rceil}\sum_{t\in \mathcal{F}_{i,\tau}^\complement}\mathds{1}\{\mathcal{C}1_j\}\mathbb{E}\left[\left(\frac{1}{p_{i^*(t),t,\tau}^i}-1\right)\bigg|\mathcal{C}1\right]}_{(B)}\right], \label{eq:scomp}
\end{align}
notice that, for each $j$, the only summands that will contribute to the sum will be those for which condition $\mathcal{C}1_j$ holds true. Thus, for each $j$, the following will hold:
\begin{align}
    \sum_{t\in \mathcal{F}_{i,\tau}^\complement}\mathds{1}\{\mathcal{C}1_j\}\mathbb{E}\left[\left(\frac{1}{p_{i^*(t),t,\tau}^i}-1\right)\bigg|\mathcal{C}1\right]&
    =\sum_{t\in \mathcal{F}_{i,\tau}^\complement}\mathds{1}\{\mathcal{C}1_j\}\underbrace{\mathbb{E}\left[\left(\frac{1}{p_{i^*(t),t,\tau}^i}-1\right)\bigg|\ \mathcal{C}1_j\right]  }_{(*)} \label{eq:po}.
\end{align}
We are now interested in evaluating $(*)$ for each $j$. For this purpose we rewrite it as:
\begin{align}
    (*)=\mathbb{E}_{N_j',\ \underline{\mu}_{i^*(t)}}\left[\underbrace{\mathbb{E}\left[\left(\frac{1}{p_{i^*(t),t,\tau}^i}-1\right)\bigg | 
\ \mathcal{C}1_j, \ N_{i^*(t),t,\tau}=N_j',\ \underline{\mu}_{i^*(t)}\right]}_{(*')}\right],
\end{align}
where the expected value $\mathbb{E}_{N_j',\ \underline{\mu}_{i^*(t)}}[\cdot]$ is taken over all the values of $N_{j-1}<N_j'\leq N_j$ (and over all different histories $\underline{\mu}_{i^*(t)}$ that yield to $N_j'$ trials, with $\underline{\mu}_{i^*(t)}$ being the set of the $N_j'$ probabilities of success of every trial of the best arm) that make $\mathcal{C}_{1,j}$ true. Notice that, given the number of plays $N'_j$ (Bernoulli trials) of the best arm, the number of successes of those trials will be distributed as a Poisson-Binomial distribution (\cite{wang1993poissobinomial}), i.e., by the distribution describing the probability of successes of $N_j'$ Bernoulli trials with different probability of success.
In order to bound these terms, we remember that $p_{i^*(t),t,\tau}^i=\Prob(Beta(S_{i^*(t),t,\tau}+1, F_{i^*(t),t,\tau}+1)>y_{i,t}|\mathcal{F}_{t-1})=F^B_{N_j'+1,y_{i,t}}(S_{i^*(t),t,\tau})$ (where the equality follows from Lemma \ref{lem:betabin}), exploiting Lemma \ref{lemma:tech} we infer that any bound obtained for the stationary case (that is when the sum of successes given $N_j'$ trials is given by a Binomial distribution) on the term $(*')$ will also hold true for the non-stationary case, then we can bound $(*')$ with Lemma 4 by \cite{agrawal2012analysis}, using as the average reward for the best arm the smaller possible average reward within the time window $\tau$ (i.e., $\min_{t'\in \dsb{t-\tau,t-1}}{\mu_{i^*(t),t'}}$) that, as encoded by Lemma \ref{lemma:tech}, is the worst case scenario for the quantity under analysis. Let $f_{N_j',\ \underline{\mu}_{i^*(t)}}(s)$ the probability mass function for the Poisson-Binomial distribution after $N_j'$ trials (each with different probability of success encoded by the set of $N_j'$ elements $\underline{\mu}_{i^*(t)}$), considered in $s$ and similarly $f_{N_j',\mu}(s)$, the probability mass function for a Binomial distribution with parameters $N_j'$ and $\mu$ considered in $s$, for ease of notation we will denote $\mu_{i^*}'\coloneqq\min_{t'\in \dsb{t-\tau,t-1}}{\mu_{i^*(t),t'}}$, by Lemma \ref{lemma:tech} holds:
\begin{align}
    (*')=\sum_{s=0}^{N_j'} \frac{f_{N_j',\underline{\mu}_{i^*(t)}}(s)}{F_{N_j'+1,y_{i,t}}(s)}-1 &\leq \sum_{s=0}^{N_j'}\frac{f_{N_j',\mu_{i^*}'}(s)}{F_{N_j'+1,y_{i,t}}(s)}-1\nonumber\\
    &\leq \begin{cases}
     O\left(\frac{1}{\Delta_i''}\right)  &\textit{ if  }  N_j'< \frac{8}{\Delta_i''}\\  
     \\
     O\left(e^{-\frac{\Delta_i''^2 N_j'}{2}}+\frac{e^{-D_{i,t}N_j'}}{N_j'\Delta_i''^2}+\frac{1}{e^{\Delta_i'^2\frac{N_j'}{4}}-1}\right) &\textit{ if  } N_j'\ge\frac{8}{\Delta''_i}  \\ 
    \end{cases},\label{eq:pare}\\
    &\le \begin{cases}
        O\left(\frac{1}{\Delta_i''} \right)&\textit{ if  }  N_j'< \frac{8}{\Delta_i''}\\
        O\left(\frac{2}{\Delta_i''^2N_j'}+\frac{1}{\Delta_i''^2N_j'}+\frac{4}{\Delta_i''^2N_j'}\right) &\textit{ if  }  N_j'\ge \frac{8}{\Delta_i''}
    \end{cases},\label{eq:pare2}\\
    &= \begin{cases}
        O\left(\frac{1}{\Delta_i''} \right)&\textit{ if  }  N_j'< \frac{8}{\Delta_i''}\\
        O\left(\frac{1}{\Delta_i''^2N_j'}\right) &\textit{ if  }  N_j'\ge \frac{8}{\Delta_i''}
    \end{cases}.
\end{align}
where by definition $\Delta_i'':=(\mu_{i^*}'-y_{i,t})$ and $D_{i,t}\coloneqq y_{i,t}\log{\frac{y_{i,t}}{\mu_{i^*}'}}+(1-y_{i,t})\log{\frac{1-y_{i,t}}{1-\mu_{i^*}'}}$. Where inequality in Equation \eqref{eq:pare} follows from Lemma 4 of  \cite{agrawal2017near}, while the inequalities from Equation \eqref{eq:pare} to Equation \eqref{eq:pare2} follow from the facts that $e^{-x}\le\frac{1}{x}$ (for $x\ge 0$) and $e^x\ge 1+x$ (for every value of $x$). Since by definition $\Delta_i''\ge\Delta_i'$, the following will hold:
\begin{align}
    (*')=\sum_{s=0}^{N_j'}\frac{f_{N_j',\mu_{i^*}'}(s)}{F_{N_j'+1,y_{i,t}}(s)}-1 &\le \begin{cases}
        O\left(\frac{1}{\Delta_i''} \right)&\textit{ if  }  N_j'< \frac{8}{\Delta_i''}\\
        O\left(\frac{1}{\Delta_i''^2N_j'}\right) &\textit{ if  }  N_j'\ge \frac{8}{\Delta_i''}
    \end{cases},\\
    &\le\begin{cases}
        O\left(\frac{1}{\Delta_i'} \right)&\textit{ if  }  N_j'< \frac{8}{\Delta_i'}\\
        O\left(\frac{1}{\Delta_i'^2N_j'}\right) &\textit{ if  }  N_j'\ge \frac{8}{\Delta_i'}
    \end{cases},\\
    &\le \begin{cases}
        O\left(\frac{1}{\Delta_i'} \right)&\textit{ if  }  j \le \lceil \log(\frac{8}{\Delta_i'})\rceil\\
        O\left(\frac{1}{\Delta_i'^2N_j'}\right) &\textit{ if  }  j\ge \lceil\log(\frac{8}{\Delta_i'})\rceil+1
        \end{cases},\\
        &\le \begin{cases}
        O\left(\frac{1}{\Delta_i'} \right)&\textit{ if  }  j \le \lceil \log(\frac{8}{\Delta_i'})\rceil\\
        O\left(\frac{1}{\Delta_i'^2N_{j-1}}\right) &\textit{ if  }  j\ge \lceil\log(\frac{8}{\Delta_i'})\rceil+1
        \end{cases},
\end{align}

where the last inequality follows as by definition, for every $j$, holds that $N_{j-1}<N_j'$ First, we face all the terms such that $j\in [1,\lceil \log(\frac{8}{\Delta_i'})\rceil]$, i.e., term $(A)$ in Equation \eqref{eq:scomp}. Notice that $\Delta_i'$ does not depend neither on $N_j'$ nor on $\underline{\mu}_{i^*(t)}$, so that we can write:
\begin{align}
    (A)&\le O\left(\frac{1}{\Delta_i'}\sum_{j=1}^{\lceil\log(\frac{8}{\Delta_i'})\rceil}\sum_{t\in\mathcal
    F_{i,\tau}^\complement} \mathds{1}\{\mathcal{C}1_j\}\right)\label{eq:pfp}\\
    &\le O\left(\frac{1}{\Delta_i'}\sum_{t \in \mathcal{F}_{i,\tau}^\complement}\mathds{1}\left\{I_t=i^*(t),\ N_{i^*(t),t,\tau}\le \frac{8e}{\Delta_i'}\right\}\right)\label{eq:pfp2}\\
    &\le O\left(\frac{1}{\Delta_i'}\frac{8eT}{\tau\Delta_i'}\right)=O\left(\frac{T}{\tau\Delta_i'^2}\right)    ,
\end{align}
where the inequality from Equation \eqref{eq:pfp} to Equation \eqref{eq:pfp2} follows from the fact that by definition the following will hold: $\sum_{j=1}^{\lceil\log(\frac{8}{\Delta_i'})\rceil}\mathds{1}\{\mathcal{C}1_j\}=\mathds{1}\left\{I_t=i^*(t), N_{i^*(t),t,\tau}\le e^{\lceil\log(\frac{8}{\Delta_i'})\rceil}\right\}$; while the last inequality is derived by Lemma \ref{lemma:window}.
We face now those term such that $j\in [\lceil\log(\frac{8}{\Delta_i'})\rceil+1,\lceil\log(\tau)\rceil]$, term $(B)$ in \eqref{eq:scomp}. Yet again, given $j$, $\frac{1}{\Delta_i'N_{j-1}}$ does not depend on neither $N_j'$ nor $\underline{\mu}_{i^*(t)}$, so we can write:
\begin{align}
    (B)&\le O\left(\frac{1}{\Delta_i'^2}\sum_{j=\lceil\log(\frac{8}{\Delta_i'})\rceil+1}^{\lceil\log(\tau)\rceil}\frac{1}{N_{j-1}}\sum_{t\in\mathcal{F}_{i,\tau}^\complement}\mathds{1}\left\{I_t=i^*(t),N_{j-1}<N_{i^*(t),t,\tau}\le N_j\right\}\right)\\
    &\le O\left(\frac{1}{\Delta_i'^2}\sum_{j=\lceil\log(\frac{8}{\Delta_i'})\rceil+1}^{\lceil\log(\tau)\rceil}\frac{1}{N_{j-1}}\sum_{t\in\mathcal{F}_{i,\tau}^\complement}\mathds{1}\left\{I_t=i^*(t),N_{i^*(t),t,\tau}\le N_j\right\}\right)\label{eq:w}\\
    &\le O\left(\frac{1}{\Delta_i'^2}\sum_{j=\lceil\log(\frac{8}{\Delta_i'})\rceil+1}^{\lceil\log(\tau)\rceil}\frac{1}{N_{j-1}}\frac{N_jT}{\tau}\right)\label{eq:w2}\\
    &\le O\left(\frac{eT}{\Delta_i'^2\tau}\lceil\log(\tau)\rceil\right).
\end{align}
The inequality from Equation \eqref{eq:w} to Equation \eqref{eq:w2} follows again from Lemma \ref{lemma:window}, while the last inequality is derived by the fact that by definition $N_j/N_{j-1}=e$. 
We tackle now term $(\mathcal{S}.2.2)$, making the same consideration that we have done from Equation \eqref{eq:po}, we infer that the only terms that will contribute to the summands are those for which condition $\mathcal{C}2$ holds true, formally:
\begin{align}
    (\mathcal{S}.2.2)= 
    \mathbb{E}\left[\sum_{t\in \mathcal{F}^{\complement}_{i,\tau}}\mathds{1}\{\mathcal{C}2\}\underbrace{\mathbb{E}\left[\left(\frac{1}{p_{i^*(t),t,\tau}^i}-1\right) \bigg| \mathcal{C}2\right]}_{(*)}\right],
\end{align}
similarly to what we have done before, we are interested in evaluating $(*)$.
\begin{align}
    (*)=\mathbb{E}_{N',\underline{\mu}_{i^*(t)}}\left[ \underbrace{\mathbb{E}\left[\left(\frac{1}{p_{i^*(t),t,\tau}^i}-1\right)\Bigg | \ \mathcal{C}2,\ N',\ \underline{\mu}_{i^*(t)}\right]}_{(*')}\right].
\end{align}
Again, by using Lemma \ref{lemma:tech} we can bound term $(*')$ with the bounds provided in Lemma 4 in \cite{agrawal2017near} for the stationary bandit with expected reward for the best arm equal to $\mu_{i^*}'$, defined as above. Formally, since by definition of condition $\mathcal{C}2$ we have that $N'>\frac{8\log(\tau)}{\Delta_i'^2}$:
\begin{align}
    (*')&\leq \sum_{s=0}^{N'} \frac{f_{N',\mu_{i^*}'}(s)}{F_{N'+1,y_{i,t}}(s)}-1\\
    &\le O\left(e^{-\frac{\Delta_i''^2 N'}{2}}+\frac{e^{-D_{i,t}N'}}{N'\Delta_i''^2}+\frac{1}{e^{\Delta_i''^2\frac{N'}{4}}-1}\right)\label{eq:kul1}\\
    &\le O\left(e^{-4\log(\tau)}+\frac{e^{-16\log(\tau)}}{8\log(\tau)}+\frac{1}{e^{2\log(\tau)}-1}\right)\label{eq:kul2}\\
    &\le O\left(\frac{1}{\tau}\right),
\end{align}
where, from Equation \eqref{eq:kul1} to Equation \eqref{eq:kul2} we used the Pinsker's Inequality, namely: $D_{i,t}\ge2\Delta_i''^2$.
Then, summing over all rounds we get $(\mathcal{S}.2.2)\le\frac{T}{\tau}$.
The result of the statement follows by summing all the terms, remembering that  by definition $\Delta_i'=\frac{\Delta_\tau}{3}$.
\end{proof}

\swgaussgen*
\begin{proof}
     We recall Lemma \ref{lem:RDg}:
     \begin{align*}
        \mathbb{E}[N_{i,T}] & \leq |\mathcal{F}_\tau|+\frac{T}{\tau\epsilon_i}+\frac{T}{\tau}+\frac{\omega T}{\tau}+{\underbrace{\mathbb{E}\left[\sum_{t\in \mathcal{F}_{i,\tau}^\complement}\mathds{1}\left \{ p_{i,t,\tau}^i>\frac{1}{\tau\epsilon_i}, \ N_{i,t,\tau}\ge \omega\right\}\right]}_{(\mathcal{S}.1)}} +{\underbrace{\mathbb{E}\left[\sum_{t\in \mathcal{F}_{i,\tau}^\complement}\left(\frac{1}{p_{i^*(t),t,\tau}^i}-1\right)\mathds{1}\left\{I_t=i^*(t) \right\}\right]}_{(\mathcal{S}.2)}}.
    \end{align*}
     Let us define $x_{i,t}$ and $y_{i,t}$ for $t \in \mathcal{F}_{i,\tau}^{\complement}$ ($t$ being the policy-maker has to choose the arm) as:
     \begin{equation}
         \max_{t' \in \dsb{t-\tau,t-1}}\{\mu_{i(t),t'}\} <x_{i,t}<y_{i,t}<\min_{t' \in \dsb{t-\tau,t-1}}\{\mu_{i^*(t),t'}\}
     \end{equation}
     with $\Delta_{i,t,\tau}=\min_{t' \in \dsb{t-\tau,t-1}}\{\mu_{i^*(t),t'}\}-\max_{t' \in \dsb{t-\tau,t-1}}\{\mu_{i(t),t'}\}$, we consider in the following analysis the choices: 
    $$
    x_{i,t}=\max_{t' \in \dsb{t-\tau,t-1}}\{\mu_{i(t),t'}\}+\frac{\Delta_{i,t,\tau}}{3},
    $$
    $$
     y_{i,t}=\min_{t' \in \dsb{t-\tau,t-1}}\{\mu_{i^*(t),t'}\}-\frac{\Delta_{i,t,\tau}}{3}.
    $$
    Notice then that the following quantities will have their minima for those $t \in \mathcal{F}_{i,\tau}^{\complement}$ such $\Delta_{i,t,\tau}=\Delta_{\tau}$:
    \begin{equation}
        \begin{rcases}
          &y_{i,t}-x_{i,t}\\
          &x_{i,t}-\max_{t' \in \dsb{t-\tau,t-1}}\{\mu_{i(t),t'}\}\\
          &\min_{t' \in \dsb{t-\tau,t-1}}\{\mu_{i^*(t),t'}\}-y_{i,t}         
        \end{rcases}
        \text{$=\frac{\Delta_{i,t,\tau}}{3}\ge \frac{\Delta_{\tau}}{3}$,}
    \end{equation}
and independently from the time $t \in \dsb{T}$ in which happens, they will always have the same value. We refer to the minimum values the quantities above can get in 
$t\in \mathcal{F}_{\tau}^{\complement}$ as:
\begin{equation}
    \begin{rcases}
        &y_i-x_i\\
        &x_i-\mu_{i,\mathcal{F}_{\tau}^{\complement}}\\
        &\mu_{i^*,\mathcal{F}_{\tau}^{\complement}}-y_i
    \end{rcases}
    \text{$=\frac{\Delta_{\tau}}{3}$.}
\end{equation}
We choose $\omega=\frac{288\log(\tau\Delta_\tau^2+e^6)}{\gamma\Delta_\tau^2}$, $\epsilon_i=\Delta_\tau^2$, $\tau\ge e$ and $\hat{\mu}_{i,t,\tau}=\frac{S_{i,t,\tau}}{N_{i,t,\tau}}$.
\paragraph{Term ($\mathcal{S}.1$)}
Decomposing the term in two contributions, we obtain:
\begin{align}
    (\mathcal{S}.1)&= \mathbb{E}\left[\sum_{t\in \mathcal{F}_{i,\tau}^\complement}\mathds{1}\left \{ p_{i,t,\tau}^i>\frac{1}{\tau\Delta_\tau^2}, \ N_{i,t,\tau}\ge \omega\right\}\right]\\
    &\le \mathbb{E}\left[\sum_{t\in \mathcal{F}_{i,\tau}^\complement}\mathds{1}\left \{ p_{i,t,\tau}^i>\frac{1}{\tau\Delta_\tau^2}, \ N_{i,t,\tau}\ge \omega,\ \hat{\mu}_{i,t,\tau}\le x_i\right\}\right]+\mathbb{E}\left[\sum_{t\in \mathcal{F}_{i,\tau}^\complement}\mathds{1}\left \{ p_{i,t,\tau}^i>\frac{1}{\tau\Delta_\tau^2}, \ N_{i,t,\tau}\ge \omega, \ \hat{\mu}_{i,t,\tau}\ge x_i\right\}\right]\\
    &\le \underbrace{\mathbb{E}\left[\sum_{t\in \mathcal{F}_{i,\tau}^\complement}\mathds{1}\left \{ \overbrace{p_{i,t,\tau}^i>\frac{1}{\tau\Delta_\tau^2}, \ N_{i,t,\tau}\ge \omega,\ \hat{\mu}_{i,t,\tau}\le x_{i,t}}^{(*)}\right\}\right]}_{(\mathcal{S}.1.1)} + \underbrace{\sum_{t \in \mathcal{F}_{i,\tau}^\complement}\Prob(\hat{\mu}_{i,t,\tau}\ge x_{i,t}|N_{i,t,\tau}\ge \omega)}_{(\mathcal{S}.1.2)}.
\end{align}
We first tackle term $(\mathcal{S}.1.2)$, considering each summand we get:
\begin{align}
    \Prob(\hat{\mu}_{i,t,\tau}\ge x_{i,t}|N_{i,t,\tau}\ge \omega)&=\Prob(\hat{\mu}_{i,t,\tau}-\mathbb{E}[\hat{\mu}_{i,t,\tau}]\ge x_{i,t}-\mathbb{E}[\hat{\mu}_{i,t,\tau}]|N_{i,t,\tau}\ge \omega)\\
    &\le \Prob(\hat{\mu}_{i,t,\tau}-\mathbb{E}[\hat{\mu}_{i,t,\tau}]\ge x_{i}-\mu_{i,\mathcal{F}_{\tau}^\complement}|N_{i,t,\tau}\ge \omega)\label{eq:ininini}\\
    &\le e^{-\frac{1}{2\lambda^2}(x_i-\mu_{i,\mathcal{F}_\tau^\complement})^2\omega}\label{eq:ininini2}\\
    &=e^{-\frac{1}{18\lambda^2}\Delta_\tau^2\frac{288\log(\tau\Delta_\tau^2+e^6)}{\gamma\Delta_\tau^2}}\\
    &\le\frac{1}{\tau\Delta_\tau^2+e^6}.
\end{align}
Where the inequality from Equation \eqref{eq:ininini} to Equation \eqref{eq:ininini2} follows from the Chernoff bounds for subgaussian random variables, reported formally in Lemma \ref{lemma:Subg}.
Facing term $(\mathcal{S}.2.1)$, we want to evaluate if ever condition $(*)$ is met. In order to do so let us consider:
\begin{align}
    \Prob\left(\mathcal{N}\left(\hat{\mu}_{i,t,\tau},\frac{1}{\gamma N_{i,t,\tau}}\right)>y_{i,t}\bigg | N_{i,t,\tau}\ge \omega, \ \hat{\mu}_{i,t,\tau}\le x_{i,t}, \ \mathbb{F}_{t-1}\right)&\le \Prob\left(\mathcal{N}\left(x_{i,t},\frac{1}{\gamma N_{i,t,\tau}}\right)>y_{i,t}\bigg | N_{i,t,\tau}\ge \omega\right),\label{eq:normin}
\end{align}

where the inequality in Equation \eqref{eq:normin} follows from Lemma \ref{lem:nbord}.
Using Lemma \ref{lemma:Abramowitz2}:
\begin{align}
\Prob\left(\mathcal{N}\left(x_{i,t}, \frac{1}{\gamma N_{i,t,\tau}}\right)>y_{i,t}\right) & \leq \frac{1}{2} e^{-\frac{\left(\gamma N_{i,t,\tau}\right)\left(y_{i,t}-x_{i,t}\right)^2}{2}} \\ & \leq \frac{1}{2} e^{-\frac{\left(\gamma \omega\right)\left(y_{i}-x_{i}\right)^2}{2}},
\end{align}
which is smaller than $\frac{1}{\tau \Delta_{\tau}^2}$ because $\omega \geq \frac{2 \ln \left(\tau\Delta_{\tau}^2\right)}{\gamma \left(y_i-x_i\right)^2}$. Substituting, we get:
\begin{equation}
   \Prob\left(\theta_{i,t,\tau}>y_{i,t} \mid N_{i,t,\tau}\ge\omega, {\hat{\mu}}_{i,t,\tau} \leq x_{i,t}, \mathbb{F}_{t-1}\right)  \leq \frac{1}{\tau\Delta_{\tau}^2}.
\end{equation}
So that condition $(*)$ is never met and $\mathcal{S}.1.1=0$.
\paragraph{Term ($\mathcal{S}.2$)}
We decompose it as:
\begin{align}\label{eq:A1}
    (\mathcal{S}.2)\le\underbrace{\mathbb{E}\left[\sum_{t\in \mathcal{F}_{i,\tau}^\complement}\left(\frac{1}{p_{i^*(t),t,\tau}^i}-1\right)\mathds{1}\left\{\overbrace{I_t=i^*(t), N_{i^*(t),t,\tau} \le \omega}^{\mathcal{C}1}\right\}\right]}_{(\mathcal{S}.2.1)}+\underbrace{\mathbb{E}\left[\sum_{t\in \mathcal{F}_{i,\tau}^\complement}\left(\frac{1}{p_{i^*(t),t,\tau}^i}-1\right)\mathds{1}\left\{\overbrace{I_t=i^*(t), N_{i^*(t),t,\tau} \ge \omega}^{\mathcal{C}2}\right\}\right]}_{(\mathcal{S}.2.2)}.
\end{align}
Let us face term $(\mathcal{S}.2.1)$. We rewrite the term, similarly to what we have done for the \texttt{Beta-TS} proof, formally:
\begin{align}
    (\mathcal{S}.2.1)=\mathbb{E}\left[\sum_{t \in \mathcal{F}_{i,\tau}^\complement}\mathds{1}\{\mathcal{C}1\}\underbrace{\mathbb{E}\left[\left(\frac{1}{p_{i^*(t),t,\tau}^i}-1\right)\Bigg | \mathcal{C}1\right]}_{(*)}\right].
\end{align}
Let us evaluate what happens when $\mathcal{C}1$ holds true, i.e., those cases in which the summands within the summation in Equation~\eqref{eq:A1} are different from zero.
We will show that whenever condition $\mathcal{C}1$ holds true $(*)$ is bounded by a constant. We will show that for any realization of the number of pulls within a time window $\tau$ such that condition $\mathcal{C}1$ holds true (i.e. number of pulls $j$ of the optimal arm within the time window less than $\omega$) the expected value of $G_j$ is bounded by a constant for all $j$ defined as earlier.
Let $\Theta_j$ denote a $\mathcal{N}\left({\hat{\mu}}_{i^*(t),j}, \frac{1}{\gamma j}\right)$ distributed Gaussian random variable, where ${\hat{\mu}}_{i^*(t),j}$ is the sample mean of the optimal arm's rewards played $j$ times within a time window $\tau$ at time $t\in \mathcal{F}_{i,\tau}^\complement$. Let $G_j$ be the geometric random variable denoting the number of consecutive independent trials until and including the trial where a sample of $\Theta_j$ becomes greater than $y_{i,t}$. Consider now an arbitrary realization where the best arm has been played $j$ times and with sample expected rewards $\mathbb{E}[{\hat{\mu}}_{i^*(t),j}]$, respecting condition $\mathcal{C}1$  then observe that $p_{i^*(t),t,\tau}=\operatorname{Pr}\left(\Theta_j>y_{i,t} \mid \mathbb{F}_{\tau_j}\right)$ and:
\begin{align}
    \mathbb{E}\left[\frac{1}{p_{i^*(t), t,\tau}}\mid \mathcal{C}1\right]&= \mathbb{E}_{j}\left[\mathbb{E}\left[\frac{1}{p_{i^*(t), t,\tau}}\mid \mathcal{C}1, N_{i^*(t),t,\tau}=j,\ \mathbb{E}[\hat{\mu}_{i^*(t),j}]\right]\right]=\mathbb{E}_{j_{\mid \mathcal{C}1}}\left[\mathbb{E}\left[\mathbb{E}\left[G_j \mid \mathbb{F}_{\tau_j}\right]\right]\right]=\mathbb{E}_{j_{\mid \mathcal{C}1}}\left[\mathbb{E}\left[G_j\right]\right],
\end{align}
where by $\mathbb{E}_{j\mid\mathcal{C}1}[\cdot]$ we denote the expected value taken over every $j$ (and every possible $\mathbb{E}[{\hat{\mu}}_{i^*(t),j}]$ compatible with $j$ pulls) respecting condition $\mathcal{C}1$.
Consider any integer $r \geq 1$. Let $z=\sqrt{\ln r}$ and let random variable MAX $_r$ denote the maximum of $r$ independent samples of $\Theta_j$. We abbreviate ${\hat{\mu}}_{i^*(t),j}$ to ${\hat{\mu}}_{i^*}$ and we will abbreviate $\min_{t' \in \dsb{t-\tau,t-1}}\{\mu_{i^*(t),t'}\}$ as $\mu_{i^*}$  in the following. Then for any integer $r\ge 1$:
\begin{align}
    \Prob\left(G_j \leq r\right) & \geq \Prob\left(\operatorname{MAX}_r>y_{i,t}\right) \\ & \geq \Prob\left(\operatorname{MAX}_r>{\hat{\mu}}_{i^*}+\frac{z}{\sqrt{\gamma j}} \geq y_{i,t}\right) \\ & =\mathbb{E}\left[\mathbb{E}\left[\left.\mathds{1}\left(\operatorname{MAX}_r>{\hat{\mu}}_{i^*}+\frac{z}{\sqrt{\gamma j}} \geq y_{i,t}\right) \right\rvert\, \mathbb{F}_{\tau_j}\right]\right] \\ & =\mathbb{E}\left[\mathds{1}\left({\hat{\mu}}_{i^*}+\frac{z}{\sqrt{\gamma j}} \geq y_{i,t}\right) \Prob\left(\left.\operatorname{MAX}_r>{\hat{\mu}}_{i^*}+\frac{z}{\sqrt{\gamma j}} \right\rvert\, \mathbb{F}_{\tau_j}\right)\right].
\end{align}
For any instantiation $F_{\tau_j}$ of $\mathbb{F}_{\tau_j}$, since $\Theta_j$ is Gaussian $\mathcal{N}\left({\hat{\mu}}_{i^*}, \frac{1}{\gamma j}\right)$ distributed r.v., this gives using Lemma \ref{lemma:Abramowitz}:
\begin{align}
\Prob\left(\left.\operatorname{MAX}_r>{\hat{\mu}}_{i^*}+\frac{z}{\sqrt{\gamma j}} \right\rvert\, \mathbb{F}_{\tau_j}=F_{\tau_j}\right) & \geq 1-\left(1-\frac{1}{\sqrt{2 \pi}} \frac{z}{\left(z^2+1\right)} e^{-z^2 / 2}\right)^r \\ & =1-\left(1-\frac{1}{\sqrt{2 \pi}} \frac{\sqrt{\ln r}}{(\ln r+1)} \frac{1}{\sqrt{r}}\right)^r \\ & \geq 1-e^{-\frac{r}{\sqrt{4 \pi r \ln r}}}.
\end{align}
For $r\ge e^{12}$:
\begin{align}
   \Prob\left(\left.\operatorname{MAX}_r>{\hat{\mu}}_{i^*}+\frac{z}{\sqrt{\gamma j}} \right\rvert\, \mathbb{F}_{\tau_j}=F_{\tau_j}\right) \geq 1-\frac{1}{r^2}. 
\end{align}
Substituting we obtain:
\begin{align}
\Prob\left(G_j \leq r\right) & \geq \mathbb{E}\left[\mathds{1}\left({\hat{\mu}}_{i^*}+\frac{z}{\sqrt{\gamma j}} \geq y_{i,t}\right)\left(1-\frac{1}{r^2}\right)\right] \\ & =\left(1-\frac{1}{r^2}\right) \Prob\left({\hat{\mu}}_{i^*}+\frac{z}{\sqrt{\gamma j}} \geq y_{i,t}\right).
\end{align}
Applying \ref{lemma:Subg} to the second term, we can write:
\begin{align}
  \Prob\left({\hat{\mu}}_{i^*}+\frac{z}{\sqrt{\gamma j}} \geq \mu_{i^*}\right) \geq 1-e^{-\frac{z^2}{2\gamma \lambda^2}} \geq 1-\frac{1}{r^2},  
\end{align}
being $\gamma\leq \frac{1}{4\lambda^2}$. In fact:
\begin{align}
   \Prob\left({\hat{\mu}}_{i^*}+\frac{z}{\sqrt{\gamma j}} \leq \mu_{i^*}\right)&\leq  \Prob\left({\hat{\mu}}_{i^*}-\mathbb{E}[{\hat{\mu}}_{i^*}]+\frac{z}{\sqrt{\gamma j}} \leq \mu_{i^*}-\mathbb{E}[{\hat{\mu}}_{i^*}]\right)\\
   &\leq \Prob\left({\hat{\mu}}_{i^*}-\mathbb{E}[{\hat{\mu}}_{i^*}] \leq -\frac{z}{\sqrt{\gamma j}} \right),
\end{align}
where the last inequality follows as by definition, we will always have that $\mu_{i^*}-\mathbb{E}[{\hat{\mu}}_{i^*}]\leq 0$.
Using, $y_{i,t} \leq \mu_{i^*}$, this gives:
\begin{equation}
   \Prob\left({\hat{\mu}}_{i^*}+\frac{z}{\sqrt{\gamma j}} \geq y_{i,t}\right) \geq 1-\frac{1}{r^2} . 
\end{equation}
Substituting all back we obtain:
\begin{align} \mathbb{E}\left[G_j\right] & =\sum_{r=0}^{\infty} \Prob\left(G_j \geq r\right) \\ & =1+\sum_{r=1}^{\infty} \Prob\left(G_j \geq r\right) \\ & \leq 1+e^{12}+\sum_{r \geq 1}\left(\frac{1}{r^2}+\frac{1}{r^{2}}\right) \\ & \leq 1+e^{12}+2+2.\end{align}
This shows a constant bound independent from $j$ of $\mathbb{E}\left[\frac{1}{p_{i^*(t), t,\tau}^i}-1\right]$ for all any possible arbitrary $j$ such that condition $\mathcal{C}1$ holds true. Then:
\begin{align}
    (\mathcal{S}.2.1)&\leq (e^{12}+5)\mathbb{E}\left[\sum_{t\in \mathcal{F}^{\complement}_{\tau}}\mathds{1}\{\mathcal{C}1\}\right]\\
    &\leq (e^{12}+5)\frac{288T\ln(\tau\Delta_{\tau}^2+e^6)}{\gamma\tau\Delta_{\tau}^2},
\end{align}
where in the last inequality we exploited Lemma \ref{lemma:window} that bounds the maximum number of times $\mathcal{C}_1$ can hold true within $T$ rounds:
\begin{align}
    \sum_{t \in \mathcal{F}_{i,\tau}^{\complement}}\mathds{1}\{\mathcal{C}1\} \leq \frac{288T\ln{(\tau\Delta_{\tau}^2+e^6})}{\gamma\tau\Delta_{\tau}^2}.
\end{align}
Let us now tackle $(\mathcal{S}.2.2)$ yet again exploiting the fact that $\mathbb{E}[XY]=\mathbb{E}[X\mathbb{E}[Y\mid X]]$:
\begin{align}\label{eq:B2}
 (\mathcal{S}.2.2)=\mathbb{E}\left[\sum_{t\in \mathcal{F}^{\complement}_{i,\tau}}\mathds{1}\{\mathcal{C}2\}\underbrace{\mathbb{E}\left[\frac{1-p_{i^*(t),t,\tau}}{p_{i^*(t),t,\tau}}\mid\mathcal{C}2\right]}_{(**)}\right]   .
\end{align}
Let us evaluate what happens when $\mathcal{C}2$ holds true, that are the only cases in which the summands within the summation in Equation~\eqref{eq:B2} are different from zero.
We derive a  bound for $(**)$ for large $j$ as imposed by condition $\mathcal{C}2$. Consider then an arbitrary instantiation in which $N_{i^*(t),t,\tau}=j\ge\omega$ (as dictated by $\mathcal{C}2$):
\begin{align}
    \mathbb{E}\left[\frac{1}{p_{i^*(t), t,\tau}}\mid\mathcal{C}2\right]= \mathbb{E}_{j}&\left[\mathbb{E}\left[\frac{1}{p_{i^*(t), t,\tau}}\mid \mathcal{C}2,\ N_{i^*(t),t,\tau}=j,\ \mathbb{E}[\hat{\mu}_{i^*(t),j}]\right]\right]=\mathbb{E}_{j_{\mid \mathcal{C}2}}\left[\mathbb{E}\left[\mathbb{E}\left[G_j \mid \mathbb{F}_{\tau_j}\right]\right]\right]
    =\mathbb{E}_{j_{\mid \mathcal{C}2}}\left[\mathbb{E}\left[G_j\right]\right].
\end{align}
Where by $\mathbb{E}_{j\mid\mathcal{C}2}[\cdot]$ we denote the expected value taken over every $j$ (and possible $\mathbb{E}[\hat{\mu}_{i^*(t),j}]$ compatible with $j$ pulls) respecting condition $\mathcal{C}2$.
Given any $r \geq 1$, define $G_j, \operatorname{MAX}_r$, and $z=\sqrt{\ln r}$ as defined earlier. Again, we abbreviate ${\hat{\mu}}_{i^*(t),j}$ to ${\hat{\mu}}_{i^*}$ and we will abbreviate $\min_{t' \in \dsb{t-\tau,t-1}}\{\mu_{i^*(t),t'}\}$ as $\mu_{i^*}$  in the following. Then for any integer $r\ge 1$
\begin{align}
\Prob&\left(G_j \leq r\right)  \geq \Prob\left(\operatorname{MAX}_r>y_{i,t}\right) \\ & \geq \Prob\left(\operatorname{MAX}_r>{\hat{\mu}}_{i^*}+\frac{z}{\sqrt{\gamma j}}-\frac{\Delta_{i,t,\tau}}{6} \geq y_{i,t}\right) \\ & =\mathbb{E}\left[\mathbb{E}\left[\left.\mathds{1}\left(\operatorname{MAX}_r>{\hat{\mu}}_{i^*}+\frac{z}{\sqrt{\gamma j}}-\frac{\Delta_{i,t,\tau}}{6} \geq y_{i,t}\right) \right\rvert\, \mathbb{F}_{\tau_j}\right]\right] \\ & =\mathbb{E}\left[\mathds{1}\left({\hat{\mu}}_{i^*}+\frac{z}{\sqrt{\gamma j}}+\frac{\Delta_{i,t,\tau}}{6} \geq \mu_{i^*}\right) \Prob\left(\left.\operatorname{MAX}_r>{\hat{\mu}}_{i^*}+\frac{z}{\sqrt{\gamma j}}-\frac{\Delta_{i,t,\tau}}{6} \right\rvert\, \mathbb{F}_{\tau_j}\right)\right],
\end{align}
where we used that $y_{i,t}=\mu_{i^*}-\frac{\Delta_{i,t,\tau}}{3}$. Now, since $j \geq \omega=\frac{288 \ln \left(\tau\Delta_{\tau}^2+e^{6}\right)}{\gamma\Delta_{\tau}^2}\ge \frac{288 \ln \left(\tau\Delta_{i,t,\tau}^2+e^{6}\right)}{\gamma(\Delta_{i,t,\tau})^2}$ for $t \in \mathcal{F}_{\tau}$, as $\Delta_{i,t,\tau}\ge \Delta_{\tau}$, we have that:
\begin{align}
    2 \frac{\sqrt{2 \ln \left(\tau\Delta_{i,t,\tau}^2+e^{6}\right)}}{\sqrt{\gamma j}} \leq \frac{\Delta_{i,t,\tau}}{6}.
\end{align}
Therefore, for $r \leq\left(\tau \Delta_{i,t,\tau}^2+e^{6}\right)^2$:
\begin{align}
 \frac{z}{\sqrt{\gamma j}}-\frac{\Delta_{i,t,\tau}}{6}=\frac{\sqrt{\ln (r)}}{\sqrt{\gamma j}}-\frac{\Delta_{i,t,\tau}}{6} \leq-\frac{\Delta_{i,t,\tau}}{12}.
\end{align}

Then, since $\Theta_j$ is $\mathcal{N}\left({\hat{\mu}}_{i^*,j}, \frac{1}{\gamma j}\right)$ distributed random variable, using the upper bound in Lemma \ref{lemma:Abramowitz2}, we obtain for any instantiation $F_{\tau_j}$ of history $\mathbb{F}_{\tau_j}$,
\begin{align}
    \Prob\left(\left.\Theta_j>{\hat{\mu}}_{i^*}-\frac{\Delta_{i,t,\tau}}{12} \right\rvert\, \mathbb{F}_{\tau_j}=F_{\tau_j}\right) \geq 1-\frac{1}{2} e^{-\gamma j \frac{\Delta_{i,t,\tau}^2}{288}} \geq 1-\frac{1}{2\left(\tau \Delta_{i,t,\tau}^2+e^{6}\right)},
\end{align}
being $j\geq \omega$. This implies:
\begin{align}
    \Prob\left(\left.\operatorname{MAX}_r>{\hat{\mu}}_{i^*}+\frac{z}{\sqrt{\gamma j}}-\frac{\Delta_{i,t,\tau}}{6} \right\rvert\, \mathbb{F}_{\tau_j}=F_{\tau_j}\right) \geq 1-\frac{1}{2^r\left(\tau \Delta_{i,t,\tau}^2+e^{6}\right)^r}.
\end{align}

Also, for any $t$ such condition $\mathcal{C}2$ holds true, we have $j \geq\omega$, and using \ref{lemma:Subg}, we get
\begin{align}
\Prob\left({\hat{\mu}}_{i^*}+\frac{z}{\sqrt{\gamma j}}-\frac{\Delta_ {i,t,\tau}}{6} \geq y_{i,t}\right) \geq \Prob\left({\hat{\mu}}_{i^*}\geq \mu_{i^*}-\frac{\Delta_{i,t,\tau}}{6}\right) &\geq 1-e^{- \omega \Delta_{i,t,\tau}^2 / 72\lambda^2}\label{eq:9999}\\ &\geq 1-\frac{1}{\left(\tau\Delta_{i,t,\tau}^2+e^{6}\right)^{16}},
\end{align}
where the last inequality of Equation~\eqref{eq:9999} follows from the fact that:
\begin{align}
    \Prob\left({\hat{\mu}}_{i^*}\geq \mu_{i^*}-\frac{\Delta_{i,t,\tau}}{6}\right)&\ge 1-\Prob\left({\hat{\mu}}_{i^*}\leq \mu_{i^*}-\frac{\Delta_{i,t,\tau}}{6}\right) \\
    &\ge 1-\Prob\left({\hat{\mu}}_{i^*}-\mathbb{E}[{\hat{\mu}}_{i^*}]\leq \mu_{i^*}-\mathbb{E}[{\hat{\mu}}_{i^*}]-\frac{\Delta_{i,t,\tau}}{6}\right) \\
    &\ge 1-\Prob\left({\hat{\mu}}_{i^*}-\mathbb{E}[{\hat{\mu}}_{i^*}]\leq-\frac{\Delta_{i,t,\tau}}{6}\right),
\end{align}
where the last inequality follows as by definition, we will always have that $\mu_{i^*}-\mathbb{E}[{\hat{\mu}}_{i^*}]\leq 0$.

Let $T^{\prime}=\left(\tau \Delta_{i,t,\tau}^2+e^{6}\right)^2$. Therefore, for $1 \leq r \leq T^{\prime}$
\begin{align}
    \Prob\left(G_j \leq r\right) \geq 1-\frac{1}{2^r\left(T^{\prime}\right)^{r / 2}}-\frac{1}{\left(T^{\prime}\right)^8}.
\end{align}

When $r \geq T^{\prime} \geq e^{12}$, we obtain:
\begin{align}
    \Prob\left(G_j \leq r\right) \geq 1-\frac{1}{r^2}-\frac{1}{r^{2}}.
\end{align}

Combining all the bounds, we have derived a bound independent from $j$ as:
\begin{align}
\mathbb{E}\left[G_j\right] & \leq \sum_{r=0}^{\infty} \Prob\left(G_j \geq r\right) \\ & \leq 1+\sum_{r=1}^{T^{\prime}} \Prob\left(G_j \geq r\right)+\sum_{r=T^{\prime}}^{\infty} \Prob\left(G_j \geq r\right) \\ & \leq 1+\sum_{r=1}^{T^{\prime}} \frac{1}{\left(2 \sqrt{T^{\prime}}\right)^r}+\frac{1}{\left(T^{\prime}\right)^7}+\sum_{r=T^{\prime}}^{\infty} \frac{1}{r^2}+\frac{1}{r^{1.5}} \\ & \leq 1+\frac{1}{\sqrt{T^{\prime}}}+\frac{1}{\left(T^{\prime}\right)^7}+\frac{2}{T^{\prime}}+\frac{3}{\sqrt{T^{\prime}}} \\ & \leq 1+\frac{5}{\tau \Delta_{i,t,\tau}^2+e^{6}}\leq 1+\frac{5}{\tau \Delta_{\tau}^2+e^{6}}.
\end{align}
So that:
\begin{align}
    (\mathcal{S}.2.2)\leq \frac{5T}{(\tau\Delta_{\tau}^2+e^6)}\leq \frac{5T}{\tau\Delta_{\tau}^2}.
\end{align}
The statement follows by summing all the terms.
\end{proof}

\restlessbeta*
\begin{proof}
    The proof follows by defining $\mathcal{F}_{\tau}$ as the set of times of length $\tau$ after every breakpoint, and noticing that by definition of the general abruptly changing setting, we have for any $t \in \mathcal{F}_{\tau}^{\complement}$, as we have demonstrated in the main paper, that:
    $$
   \min_{t' \in \dsb{t-\tau,t-1}}\{\mu_{i^*(t),t'}\} > \max_{t' \in \dsb{t-\tau,t-1}}\{\mu_{i(t),t'}\}.
    $$
\end{proof}
\restlessgauss*
\begin{proof}
    The proof, yet again, follows by defining $\mathcal{F}_{\tau}$ as the set of times of length $\tau$ after every breakpoint, and noticing that by definition of the general abruptly changing setting we have for any $t \in \mathcal{F}_{\tau}^{\complement}$, as we have demonstrated in the main paper, that:
    $$
   \min_{t' \in \dsb{t-\tau,t-1}}\{\mu_{i^*(t),t'}\} > \max_{t' \in \dsb{t-\tau,t-1}}\{\mu_{i(t),t'}\}.
    $$

\end{proof}

\swbetasmooth*
\begin{proof}
    To derive the bound, we will assign "error" equal to one for every $t \in \mathcal{F}_{\Delta', T}$ and we will study what happens in  $\mathcal{F}_{\Delta',T}^{\complement}$. Notice that by definition of $\mathcal{F}_{\Delta',T}^{\complement}$ we will have that $\forall i \neq i^*(t)$:
    $$
    \mu_{i^*(t),t}-\mu_{i,t}\ge \Delta'>2\sigma\tau.
    $$
    Using the Lipsitchz assumption we can infer that for $i \neq i^*(t)$:
    $$
    \min_{t'\in \dsb{t-\tau,t-1}}\{\mu_{i^*(t),t'}\}\ge \mu_{i^*(t),t}-\sigma\tau,
    $$
    and, similarly, by making use of the Lipscithz assumption, we obtain, for $i \neq i^*(t)$:
    $$
    \max_{t'\in \dsb{t-\tau,t-1}}\{\mu_{i,t'}\}\leq \mu_{i,t}+\sigma\tau.
    $$
    Substituting we obtain:
    $$
    \min_{t'\in \dsb{t-\tau,t-1}}\{\mu_{i^*(t),t'}\}-\max_{t'\in \dsb{t-\tau,t-1}}\{\mu_{i,t'}\}\ge \mu_{i^*(t),t}-\sigma\tau-\mu_{i,t}-\sigma\tau,
    $$
    so that due to the introduced assumptions, we have:
    $$
    \min_{t'\in \dsb{t-\tau,t-1}}\{\mu_{i^*(t),t'}\}-\max_{t'\in \dsb{t-\tau,t-1}}\{\mu_{i,t'}\}\ge \Delta'-2\sigma\tau>0.
    $$
    Notice that is the assumption for the general theorem, so we will have that $\mathcal{F}_{\Delta', T}^{\complement} = \mathcal{F}_{\tau}^{\complement}$, this yields to the desired result noticing that by definition $\Delta_{\tau}=\Delta'-2\sigma\tau$.
\end{proof}

\swgtssmooth*
\begin{proof}
       In order to derive the bound we will assign "error" equal to one for every $t\in\mathcal{F}_{\Delta',T}$ and we will study what happens in  $\mathcal{F}_{\Delta',T}^{\complement}$, i.e. the set of times $t \in \dsb{T}$ such that $t \notin \mathcal{F}_{\Delta',T}$. Notice that by definition of $\mathcal{F}_{\Delta',T}^{\complement}$ we will have that $\forall i \neq i^*(t)$:
    $$
    \mu_{i^*(t),t}-\mu_{i,t}\ge \Delta'>2\sigma\tau.
    $$
    Using the Lipschitz assumption, we can infer that for $i^*(t)$:
    $$
    \min_{t'\in \dsb{t-\tau,t-1}}\{\mu_{i^*(t),t'}\}\ge \mu_{i^*(t),t}-\sigma\tau,
    $$
    and, similarly, using the Lipschitz assumption, we obtain, for $i \neq i^*(t)$:
    $$
    \max_{t'\in \dsb{t-\tau,t-1}}\{\mu_{i,t'}\}\leq \mu_{i,t}+\sigma\tau.
    $$
    Substituting we obtain:
    $$
    \min_{t'\in \dsb{t-\tau,t-1}}\{\mu_{i^*(t),t'}\}-\max_{t'\in \dsb{t-\tau,t-1}}\{\mu_{i,t'}\}\ge \mu_{i^*(t),t}-\sigma\tau-\mu_{i,t}-\sigma\tau,
    $$
    so that due to the introduced assumptions, we have:
    $$
    \min_{t'\in \dsb{t-\tau,t-1}}\{\mu_{i^*(t),t'}\}-\max_{t'\in \dsb{t-\tau,t-1}}\{\mu_{i,t'}\}\ge \Delta'-2\sigma\tau>0.
    $$
    Notice that is the assumption for the general theorem, so we will have that $\mathcal{F}_{\Delta',T}^{\complement}=\mathcal{F}_{\tau}^{\complement}$, this yields to the desired result noticing that by definition $\Delta_{\tau}=\Delta'-2\sigma\tau$.
\end{proof}

\section{Experimental details}\label{apx:experiments}
\subsection*{Parameters}
The choices of the parameters of the algorithms we compared \texttt{R-less/ed-UCB} with are the following:
\begin{itemize}
	\item \texttt{Rexp3}: $\gamma = \min \left\{ 1, \sqrt{\frac{K\log{K}}{(e-1)\Delta_T}} \right\}$, $\Delta_T = \lceil (K\log{K})^{1/3} (T/V_T)^{2/3} \rceil$ as recommended by~\citet{besbes2014stochastic};
	\item \texttt{KL-UCB}: $c = 3$ as required by the theoretical results on the regret provided by~\citet{garivier2011kl};
	\item \texttt{Ser4}: according to what suggested by~\citet{allesiardo2017non} we selected $\delta=1/T$, $\epsilon=\frac{1}{KT}$, and $\phi = \sqrt{\frac{N}{TK\log({KT})}}$;
	\item \texttt{SW-UCB}: as suggested by~\citet{garivier2008upper} we selected the sliding-window $\tau = 4\sqrt{T\log{T}}$ and the constant $\xi = 0.6$;
	\item \texttt{SW-KL-UCB} as suggested by~\citet{garivier2011upper} we selected the sliding-window $\tau = \sigma^{-4/5}$;
	\end{itemize}
\subsection*{Equations for the Abruptly Changing Environment}
\begin{equation}
    \bm{\mu}=\begin{cases}
        \mu_{i,t}= 0.2+0.05(i-2) \hspace{0.1 cm}\textit{ if $i\in \{2,\ldots,8\}$} \\
        \\
        \mu_{1,t}=\begin{cases}
        0.1 \hspace{0.1 cm}\textit{ if t<15000 or 30000<t<40000 }\\
        0.99 \hspace{0.1 cm}\textit{ otherwise}
        \end{cases}\\
        \\
        \mu_{9,t}=\begin{cases}
        0.55 \hspace{0.1 cm}\textit{ if t<15000 or 30000<t<40000 }\\
        0.15 \hspace{0.1 cm}\textit{ otherwise}
        \end{cases}\\
        \\
        \mu_{{10},t}=\begin{cases}
        0.6 \hspace{0.1 cm}\textit{ if t<15000 or 30000<t<40000 }\\
        0.1 \hspace{0.1 cm}\textit{ otherwise}
        \end{cases}
    \end{cases} \hspace{0.1 cm}.
    \label{eq:eqabr1}
\end{equation}
\begin{equation}
    \bm{\mu}=\begin{cases}
        \mu_{i,t}= 0.2+0.05(i-2) \hspace{0.1 cm}\textit{ if $i\in \{3,5,7,8\}$} \\
        \\
        \mu_{i,t}= 0.2+0.05(i-2) +0.1\sin{(0.001t)}\hspace{0.1 cm}\textit{ if $i\in \{2,4,6\}$} \\
        \\
        \mu_{1,t}=\begin{cases}
        0.1 \hspace{0.1 cm}\textit{ if t<15000 or 30000<t<40000 }\\
        0.9+0.1\sin{(0.001t)} \hspace{0.1 cm}\textit{ otherwise}
        \end{cases}\\
        \\
        \mu_{9,t}=\begin{cases}
        0.55 \hspace{0.1 cm}\textit{ if t<15000 or 30000<t<40000 }\\
        0.15+0.1\sin{(0.001t)} \hspace{0.1 cm}\textit{ otherwise}
        \end{cases}\\
        \\
        \mu_{10,t}=\begin{cases}
        0.6 \hspace{0.1 cm}\textit{ if t<15000 or 30000<t<40000 }\\
        0.1+0.1\sin{(0.001t)} \hspace{0.1 cm}\textit{ otherwise}
        \end{cases}
    \end{cases} \hspace{0.1 cm}.
    \label{eq:eqabr2}
\end{equation}
\subsection*{Equations for the Smoothly Changing Environment}
\begin{equation}
    \mu_{i,t}=\begin{cases}
        \frac{K-1}{K}-\frac{|w(t)-i|}{K} \\
        \\
    w(t)=1+(K-1)\frac{1+\sin(\sigma t)}{2}    
    \end{cases}
    \label{eq:smoothequation2}
\end{equation}
\subsection*{Smoothly Changing Experiment for $\sigma=0.001$}
The environment is illustrated in Figure \ref{fig:smoothsettingfreq}. The cumulative regret is depicted in Figure \ref{fig:smoothexperimentfreq}, while the sensitivity analysis is represented in Figure \ref{fig:smoothsensfreq}.
\begin{figure}
\centering
\subfloat[]{
\scalebox{1}{\includegraphics[]{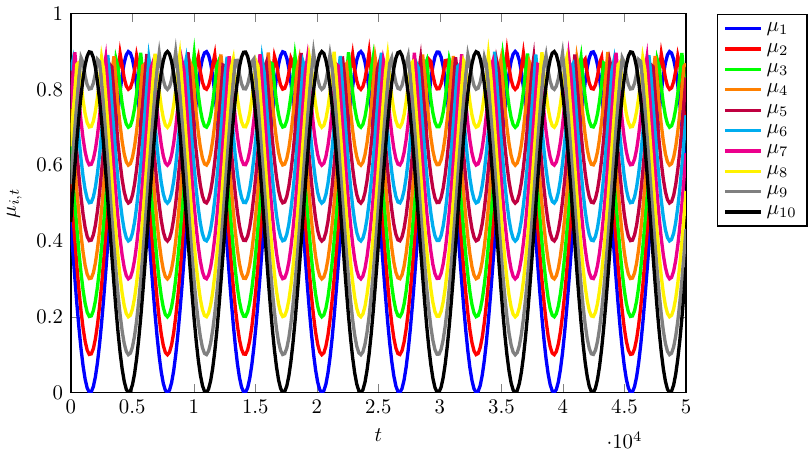}}
\label{fig:smoothsettingfreq}
}\\
\subfloat[]{
\scalebox{1}{\includegraphics[]{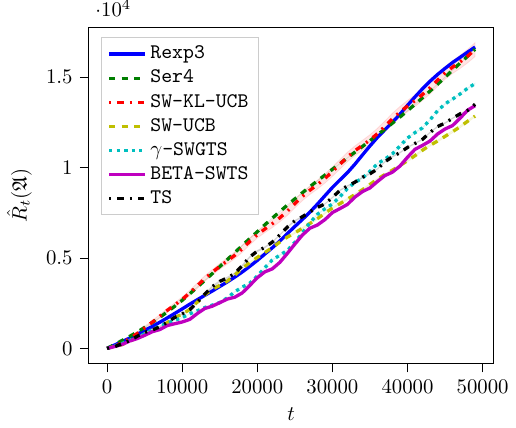}}
\label{fig:smoothexperimentfreq}
}
\subfloat[]{
    \scalebox{1}{\includegraphics[]{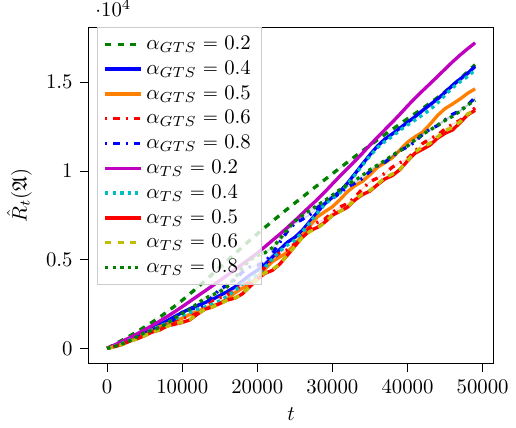}}
    \label{fig:smoothsensfreq}
}
\caption{10 arms experiment: (a) the smoothly changing environment with $\sigma=0.001$, (b) cumulative regret comparison, (c) sensitivity analysis for the sliding window size.}
\label{fig:smoothlytwo}
\end{figure}

\section{Errors from the paper by~\citet{trovo2020sliding}}\label{apx:trovo}
In this appendix, we report the technical error found in \citet{trovo2020sliding}.
Rewriting Equation (18) to Equation (21) from \cite{trovo2020sliding}:

\begin{align} R_A & =\sum_{t \in \mathcal{F}_\phi^{\prime}} \mathbb{P}\left(\vartheta_{i_\phi^*, t} \leq \mu_{i_\phi^*, t}-\sqrt{\frac{5 \log \tau}{T_{i_\phi^*, t, \tau}}}\right) \\ & \leq \sum_{t \in \mathcal{F}_\phi^{\prime}} \mathbb{P}\left(\vartheta_{i_\phi^*, t} \leq \mu_{i_\phi^*, t}-\sqrt{\frac{5 \log \tau}{T_{i_\phi^*, t, \tau}}}, T_{i_\phi^*, t, \tau}>\bar{n}_A\right)+\hspace{-0.15 cm}\sum_{t \in \mathcal{F}_\phi^{\prime}} \mathbb{P}\left(T_{i_\phi^*, t, \tau} \leq \bar{n}_A\right) \\ & \leq \sum_{t \in \mathcal{F}_\phi^{\prime}} \mathbb{P}\left(\vartheta_{i_\phi^*, t} \leq \mu_{i_\phi^*, t}-\sqrt{\frac{5 \log \tau}{T_{i_\phi^*, t, \tau}}}, T_{i_\phi^*, t, \tau}>\bar{n}_A\right)+\sum_{t \in \mathcal{F}_\phi^{\prime}} \mathbb{E}\left[\mathds{1}\left\{T_{i_\phi^*, t, \tau} \leq \bar{n}_A\right\}\right] \\ & \leq \sum_{t \in \mathcal{F}_\phi^{\prime}} \mathbb{P}\left(\vartheta_{i_\phi^*, t} \leq \mu_{i_\phi^*, t}-\sqrt{\frac{5 \log \tau}{T_{i_\phi^*, t, \tau}}}, T_{i_\phi^*, t, \tau}>\bar{n}_A\right)+\bar{n}_A \frac{N_\phi}.{\tau}\end{align}

Notice that the term $\sum_{t\in \mathcal{F}_{\phi}'}\mathbb{E}\left[\mathds{1}\left\{T_{i_\phi^*, t, \tau} \leq \bar{n}_A\right\}\right]$ is bounded using Lemma \ref{lemma:window}, implying that the event $\{\cdot\}$ in $\mathds{1}\{\cdot\}$ is:
\begin{align}
    \{\cdot\}=\left\{T_{i_\phi^*, t, \tau} \leq \bar{n}_A,i_t=i_{\phi}^*\right\}.
\end{align}
However, the separation of the event used by the author (following the line of proof \cite{kaufmann2012thompson}) in Equation (12) to Equation (16) in \cite{trovo2020sliding}:
\begin{align}
 \mathbb{E}\left[T_i\left(\mathcal{F}_\phi^{\prime}\right)\right]& =\sum_{t \in \mathcal{F}_\phi^{\prime}} \mathbb{E}\left[\mathds{1}\left\{i_t=i\right\}\right] \\
& =\hspace{-0.15 cm}\sum_{t \in \mathcal{F}_\phi^{\prime}}\left[\mathbb{P}\left(\vartheta_{i_\phi^*, t} \leq \mu_{i_\phi^*, t}-\sqrt{\frac{5 \log \tau}{T_{i_\phi^*}, t, \tau}}, i_t=i\right)\hspace{-0.1 cm}+\hspace{-0.1 cm}\mathbb{P}\left(\vartheta_{i_\phi^*, t}>\mu_{i_\phi^*, t}-\sqrt{\frac{5 \log \tau}{T_{i_\phi^*, t, \tau}}}, i_t=i\right)\right] \\
& \leq \sum_{t \in \mathcal{F}_\phi^{\prime}} \hspace{-0.1 cm}\mathbb{P}\left(\vartheta_{i_\phi^*, t} \leq \mu_{i_\phi^*, t}-\sqrt{\frac{5 \log \tau}{T_{i_\phi^*, t, \tau}}},i_t=i\right)+\hspace{-0.15 cm}\sum_{t \in \mathcal{F}_\phi^{\prime}} \mathbb{P}\left(\vartheta_{i, t}>\mu_{i_\phi^*, t}-\sqrt{\frac{5 \log \tau}{T_{i_\phi^*}, t, \tau}}, i_t=i\right) \\
& \leq \sum_{t \in \mathcal{F}_\phi^{\prime}} \mathbb{P}\left(\vartheta_{i_\phi^*, t} \leq \mu_{i_\phi^*, t}-\sqrt{\frac{5 \log \tau}{T_{i_\phi^*, t, \tau}}}, i_t=i\right)+ \nonumber\\
& +\sum_{t \in \mathcal{F}_\phi^{\prime}} \mathbb{P}\left(\vartheta_{i, t}>\mu_{i_\phi^*, t}-\sqrt{\frac{5 \log \tau}{T_{i_\phi^*, t, \tau}}}, i_t=i, \vartheta_{i, t}<q_{T_{i, t, \tau}}\right)+\sum_{t \in \mathcal{F}_\phi^{\prime}} \mathbb{P}\left(\vartheta_{i, t} \geq q_{T_{i, t, \tau}}\right) \\
& \leq \hspace{-0.15 cm}\underbrace{\sum_{t \in \mathcal{F}_\phi^{\prime}} \hspace{-0.1 cm}\mathbb{P}\left(\vartheta_{i_\phi^*, t} \leq \mu_{i_\phi^*, t}-\sqrt{\frac{5 \log \tau}{T_{i_\phi^*, t, \tau}}},i_t=i\right)}_{R_A}+\hspace{-0.2 cm}\underbrace{\sum_{t \in \mathcal{F}_\phi^{\prime}} \hspace{-0.1 cm}\mathbb{P}\left(u_{T_{i, t, \tau}}>\mu_{i_\phi^*, t}-\sqrt{\frac{5 \log \tau}{T_{i_\phi^*, t, \tau}}}, i_t=i\right)}_{R_B} \nonumber\\
& +\underbrace{\sum_{t \in \mathcal{F}_\phi^{\prime}} \mathbb{P}\left(\vartheta_{i, t} \geq q_{T_{i, t, \tau}}\right)}_{R_C},
\end{align}
is such that the event $\{\cdot\}$ is given by:
\begin{align}
     \{\cdot\}=\left\{T_{i_\phi^*, t, \tau} \leq \bar{n}_A,i_t=i\neq i_{\phi}^*\right\},
\end{align}
thus making the derived inequality incorrect. The same error is done also in the following equations (Equation 70 to Equation 72 in \cite{trovo2020sliding}):

\begin{align}
R_A & =\sum_{t \in \mathcal{F}_{\Delta^C, N}} \mathbb{P}\left(\vartheta_{i_t^*, t} \leq \mu_{i_t^*, t}-\sigma \tau-\sqrt{\frac{5 \log \tau}{T_{i_t^*, t, \tau}}}\right) \\
& \leq \sum_{t \in \mathcal{F}_{\Delta^C, N}} \mathbb{P}\left(\vartheta_{i_t^*, t} \leq \mu_{i_t^*, t}-\sigma \tau-\sqrt{\frac{5 \log \tau}{T_{i_t^*, t, \tau}}}, T_{i_t^*, t, \tau}>\bar{n}_A\right) \nonumber\\
& +\sum_{t \in \mathcal{F}_{\Delta^C, N}} \mathbb{P}\left(T_{i_t^*, t, \tau} \leq \bar{n}_A\right) \\
& \leq \sum_{t \in \mathcal{F}_{\Delta C, N}} \mathbb{P}\left(\vartheta_{i_t^*, t} \leq \mu_{i_t^*, t}-\sigma \tau-\sqrt{\frac{5 \log \tau}{T_{i_t^*, t, \tau}}}, T_{i_t^*, t, \tau}>\bar{n}_A\right)+\bar{n}_A\left\lceil\frac{N}{\tau}\right\rceil,
\end{align}
where notice that yet again $\sum_{t \in \mathcal{F}_{\Delta^C, N}} \mathbb{P}\left(T_{i_t^*, t, \tau} \leq \bar{n}_A\right)$ has been wrongly bounded by $\bar{n}_A\lceil\frac{N}{\tau}\rceil$.

\section{Auxiliary Lemmas}
In this appendix, we report some results that already exist in the bandit literature and have been used to demonstrate our results.

\begin{restatable}[Generalized Chernoff-Hoeffding bound from~\cite{agrawal2017near}]{lemma}{chernoff}\label{lemma:chernoff}
Let $X_1, \ldots , X_n$ be independent Bernoulli random variables with $\mathbb{E}[X_i
] = p_i$, consider the random variable $X = \frac{1}{n}\sum_{i=1}^nX_i$, with $\mu = \mathbb{E}[X]$.
For any $0 < \lambda < 1 - \mu$ we have:
\[
\Prob(X\ge\mu+\lambda)\leq \exp{\big(-nd(\mu+\lambda,\mu)\big)},
\]
and for any $0 < \lambda < \mu$
\[
\Prob(X\leq\mu-\lambda)\leq \exp{\big(-nd(\mu-\lambda,\mu)\big)},
\]
where $d(a, b) \coloneqq a \ln{\frac{a}{b}} + (1-a) \ln{\frac{1-a}{1-b}}$.
\end{restatable}


\begin{restatable}[Beta-Binomial identity]{lemma}{betabin} \label{lem:betabin}
    For all positive integers $\alpha, \beta \in \mathbb{N}$, the following equality holds:
    \begin{equation}
        F_{\alpha, \beta}^{beta}(y) = 1 - F_{\alpha + \beta - 1, y}^B(\alpha - 1),
    \end{equation}
    where $F_{\alpha, \beta}^{beta}(y)$ is the cumulative distribution function of a beta with parameters $\alpha$ and $\beta$, and $F_{\alpha + \beta - 1, y}^B(\alpha - 1)$ is the cumulative distribution function of a binomial variable with $\alpha + \beta - 1$ trials having each probability $y$.
\end{restatable}

\begin{restatable}[\cite{abramowitz1968handbook} Formula $7.1.13$]{lemma}{Abramowitz}\label{lemma:Abramowitz}
Let $Z$ be a Gaussian random variable with mean $\mu$ and standard deviation $\sigma$, then:
\begin{equation}
    \Prob(Z>\mu+x\sigma)\ge \frac{1}{\sqrt{2\pi}}\frac{x}{x^2+1}e^{-\frac{x^2}{2}}
\end{equation}
\end{restatable}

\begin{restatable}[\cite{abramowitz1968handbook}]{lemma}{Abramowitz2}\label{lemma:Abramowitz2}
Let $Z$ be a Gaussian r.v.~with mean $m$ and standard deviation $\sigma$, then:
    \begin{equation}
        \frac{1}{4 \sqrt{\pi}} e^{-7 z^2 / 2}<\Prob(|Z-m|>z \sigma) \leq \frac{1}{2} e^{-z^2 / 2}.
    \end{equation}
\end{restatable}

\begin{restatable}[\cite{rigollet2023high} Corollary $1.7$]{lemma}{Subg}\label{lemma:Subg}
Let $X_1,\ldots, X_n$ be $n$ independent random variables such that $X_i\sim $ \textsc{Subg}($\sigma^2$), then for any $a \in \mathbb{R}^n$, we have
\begin{equation}
\Prob\left[\sum_{i=1}^n a_i X_i>t\right] \leq \exp \left(-\frac{t^2}{2 \sigma^2\|a\|_2^2}\right),    
\end{equation}

and
\begin{equation}
    \Prob\left[\sum_{i=1}^n a_i X_i<-t\right] \leq \exp \left(-\frac{t^2}{2 \sigma^2   \|a\|_2^2}\right)
\end{equation}

Of special interest is the case where $a_i=1 / n$ for all $i$ we get that the average $\bar{X}=\frac{1}{n} \sum_{i=1}^n X_i$, satisfies
$$
\Prob(\bar{X}>t) \leq e^{-\frac{n t^2}{2 \sigma^2}} \quad \text { and } \quad \mathbb{P}(\bar{X}<-t) \leq e^{-\frac{n t^2}{2 \sigma^2}}
$$
\end{restatable}

\begin{restatable}[\cite{combes2014unimodal}, Lemma D.1]{lemma}{combprot}\label{lemma:window}
     Let $A \subset \mathbb{N}$, and $\tau \in \mathbb{N}$ fixed. Define $a(n)=$ $\sum_{t=n-\tau}^{n-1} \mathds{1}\{t \in A\}$. Then for all $T \in \mathbb{N}$ and $s \in \mathbb{N}$ we have the inequality:
     \begin{align}
       \sum_{n=1}^T \mathds{1}\{n \in A, a(n) \leq s\} \leq s\lceil T / \tau\rceil .  
     \end{align}
\end{restatable}
\begin{restatable}[\citet{fiandri2025thompsonsamplinglikealgorithmsstochastic}, Lemma 5.2]{lemma}{fiandri}\label{lemma:tech}
Let $j \in \Nat$, $\mathrm{PB}(\underline{\mu}_{i^*(t)}(j))$ be a Poisson-Binomial distribution with parameters $\underline{\mu}_{i^*(t)}(j) = (\mu_{i^*(t),1},\dots,\mu_{i^*(t),j})$, and $\mathrm{Bin}(j, x)$ be a binomial distribution of $j$ trials and probability of success $0 \le x \leq \frac{1}{j} \sum_{l=1}^j \mu_{i^*(t),l} =\overline{\mu}_{i^*(t),j}$. Then, it holds that:
\begin{align*}
	& \E_{S_{i^*(t),t} \sim  \mathrm{PB}(\underline{\mu}_{i^*(t)}(j))}  \left[\frac{1}{p_{i^*(t),t}^i} \bigg| N_{i^*(t),t}=j\right]  \\&\qquad\quad \leq\E_{S_{i^*(t),t}\sim \mathrm{Bin}(j, \overline{\mu}_{i^*(t),j})} \left[\frac{1}{p_{i^*(t),t}^i} \bigg| N_{i^*(t),t}=j \right] \nonumber\\ &\qquad\quad \leq \E_{S_{i^*(t),t}\sim \mathrm{Bin}(j, x)} \left[\frac{1}{p_{i^*(t),t}^i} \bigg| N_{i^*(t),t}=j \right],
\end{align*}
where $p_{i^*(t),t}^i=\Prob(Beta(S_{i^*(t),t}+1,\ F_{i^*(t),t}+1)>y_{i,t}|\ \mathcal{F}_{t-1})$, and $S_{i^*(t),t}$, $F_{i^*(t),t}$ are respectively an arbitrary number of successes and an arbitrary number of failures after $N_{i^*(t),t}=S_{i^*(t),t}+F_{i^*(t),t}$ Bernoulli trials at time $t$. 
\end{restatable}
\begin{restatable}[Theorem $4.2.3$, Example $4.2.4$ \citet{roch2024modern}]{lemma}{bborder}\label{lem:bborder}
   Let $F_{n,p}^B$ be the CDF of a $Bin(n,p)$ distributed random variable, then holds for $m\leq n$ and $q\leq p$:
\begin{align}
   F_{n,p}^B(x)\leq F_{m,q}^B(x)
\end{align}
for all $x$.
\end{restatable}

\begin{restatable}[Beta and Normal Ordering, Lemma D.11 \cite{fiandri2025thompsonsamplinglikealgorithmsstochastic}]{lemma}{nbord}\label{lem:nbord}

    ($i$) A $\mathcal{N}\left(m, \sigma^2\right)$ distributed r.v. ~(i.e., a Gaussian random variable with mean $m$ and variance $\sigma^2$ ) is stochastically dominated by $\mathcal{N}\left(m^{\prime}, \sigma^2\right)$ distributed r.v.~if $m^{\prime} \geq m$.

    ($ii$) A $Beta(\alpha, \beta)$ random variable
is stochastically dominated by $Beta(\alpha', \beta')$ if $\alpha'\ge \alpha$ and $\beta'\leq \beta$.
\end{restatable}



\end{appendices}


\end{document}